\pdfoutput=1
\documentclass[10pt]{scrartcl}

\usepackage[english]{babel} 
\usepackage[protrusion=true,expansion=true]{microtype} 
\usepackage{amsmath,amsfonts,amssymb,amsthm,xcolor,fullpage} 
\usepackage{verbatim,fancyvrb}
\usepackage{alltt}
\usepackage{tikz}
\usepackage{graphicx}
\usepackage{enumerate}
\usepackage{caption}
\usepackage{subcaption}
\usepackage[]{algorithm}
\usepackage{algpseudocode}
\usepackage{stmaryrd}
\usepackage{afterpage}
\usepackage{url}
\usepackage[english=american]{csquotes}
\usepackage{booktabs}
\usepackage{capt-of}
\usepackage[page]{appendix}

\newtheorem{definition}{Definition}
\newtheorem{theorem}[definition]{Theorem}

\newtheorem{proposition}[definition]{Proposition}
\newtheorem{lemma}[definition]{Lemma}
\newtheorem{example}[definition]{Example}

\newtheorem{notation}[definition]{Notation}

\def\B{{\mathbb B}}
\def\N{{\mathbb N}}

\def\R{{\mathbb R}}

\DeclareMathOperator{\dtw}{\delta}

\newcommand{\QED}{\hfill\ensuremath{\square}}
\newcommand{\commentout}[1]{}

\newcommand{\abs}[1]{\mathop{\left\lvert #1 \right\rvert}} 
\newcommand{\args}[1]{\mathop{\left( #1 \right)}} 

\newcommand{\norm}[1]{\mathop{\left\lVert #1 \right\rVert}}
\newcommand{\cbrace}[1]{\mathop{\left\{ #1 \right\}}}

\newcommand{\argsS}[2]{\mathop{\left( #1 \right)#2}} 

\newcommand{\normS}[2]{\mathop{\left\lVert #1 \right\rVert#2}}

\renewcommand{\S}[1]{{\mathcal{#1}}}           	

\renewenvironment{cases}{%
\left\{\begin{array}{c@{\quad : \quad}l}}%
{%
\end{array}\right.}

\usepackage{booktabs}
\definecolor{gr}{rgb}{0.9,0.9,0.9}

\begin{document}

\title{\LARGE Making the Dynamic Time Warping Distance Warping-Invariant}
\author{\small Brijnesh Jain \\[-1ex]
 \small Technische Universit\"at Berlin, Germany\\[-1ex]
 \small e-mail: brijnesh.jain@gmail.com}
 
\date{}
\maketitle

\begin{abstract}
The literature postulates that the dynamic time warping (dtw) distance can cope with temporal variations but stores and processes time series in a form as if the dtw-distance cannot cope with such variations. To address this inconsistency, we first show that the dtw-distance is not warping-invariant. The lack of warping-invariance contributes to the inconsistency mentioned above and to a strange behavior. To eliminate these peculiarities, we convert the dtw-distance to a warping-invariant semi-metric, called time-warp-invariant (twi) distance. Empirical results suggest that the error rates of the twi and dtw nearest-neighbor classifier are practically equivalent in a Bayesian sense. However, the twi-distance requires less storage and computation time than the dtw-distance for a broad range of problems. These results challenge the current practice of applying the dtw-distance  in nearest-neighbor classification and suggest the proposed twi-distance as a more efficient and consistent option.
\end{abstract}

\section{Introduction}\label{sec:intro}

Computing proximities is a core task for many important time-series mining tasks such as similarity search, nearest-neighbor classification, and clustering. The choice of proximity measure affects the performance of the corresponding proximity-based method. Consequently, a plethora of task-specific proximity measures have been devised \cite{Abanda2018,Aghabozorgi2015}. 

One research direction in time series mining is devoted to constructing data-specific distance functions that can capture some of the data's intrinsic structure. For example, the \emph{dynamic time warping} (dtw) distance has been devised to account for local variations in speed \cite{Vintsyuk1968,Sakoe1978}. Further examples include distances and techniques that account for other variations such as variation in speed, amplitude, offset, endpoints, occlusion, complexity, and mixtures thereof \cite{Batista2014}. 

In some -- partly influential -- work, the different distances are ascribed as invariant under the respective variation. For example, the dtw-distance is considered as warping-invariant (invariant under variation in speed) \cite{Batista2014,Chavoshi2016,Chen2013,Jain2015,Mueen2018,Silva2016}, the CID-distance \cite{Batista2014} as complexity-invariant, and the $\psi$-dtw distance \cite{Silva2016} as endpoint-invariant. Such statements in the literature collide with the mathematical meaning of \emph{invariance}. In mathematics, a distance function $d$ is called invariant under some transformation $f$ if $d(x, y) = d(f(x), f(y))$ for all elements $x, y$ from some domain. According to the mathematical definition of invariance, the CID-distance is not complexity-invariant, the $\psi$-dtw distance is not endpoint-invariant, and in particular, the dtw-distance is not warping-invariant (see Fig.~\ref{fig:ex_condensation} and Section \ref{subsec:main-warping-invariance}). To avoid confusion due to ambiguity, we assume the mathematical definition of invariance and say instead that the dtw-distance can cope with (or accounts for)  temporal variation.

The literature postulates that the dtw-distance can cope with temporal variations but stores and processes time series in a form as if the dtw-distance cannot cope with such variations. For example, the left columns of Figure \ref{fig:ex_condensation} shows two time series $x$ and $y$ from the UCR dataset \emph{SmallKitchenAppliances} \cite{Chen2015}. Both time series have length $720$ and mostly consist of constant segments of variable length (time). If the dtw-distance can cope with temporal variations, why do we not collapse constants segments to single points? The right column of Figure \ref{fig:ex_condensation} shows the corresponding condensed forms $x^*$ and $y^*$ obtained by collapsing constant segments to singletons. The condensed forms have length $\abs{x^*} = 25$ and $\abs{y^*} = 35$, which is substantially shorter than the lengths of their original time series. As a consequence, the space-saving ratios are $\rho_s(x) = 96.5\%$ and $\rho_s(y) = 95.1\%$ and the speed-up factor of computing the dtw-distance $\delta(x^*,y^*)$ over $\delta(x, y)$ is $592.5$. 

\begin{figure}[t]
\centering
\includegraphics[width=0.85\textwidth,trim= 8cm 0cm 7cm 0cm,clip]{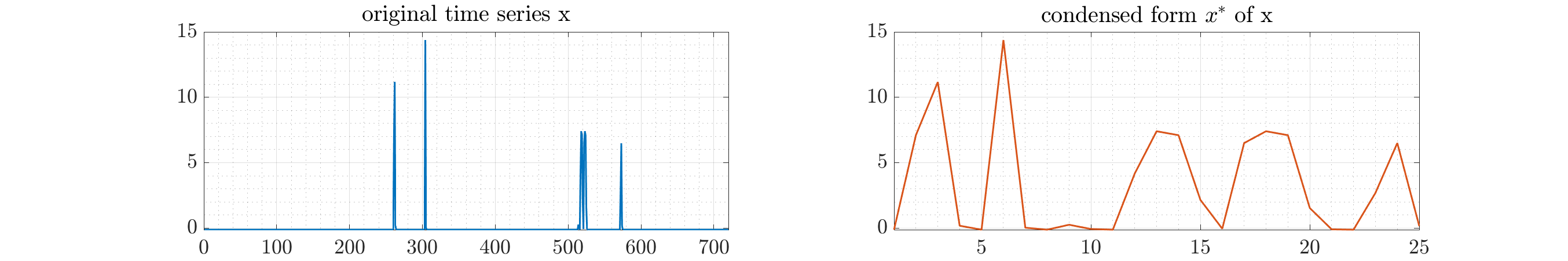}
\\[1em]
\includegraphics[width=0.85\textwidth,trim= 8cm 0cm 7cm 0cm,clip]{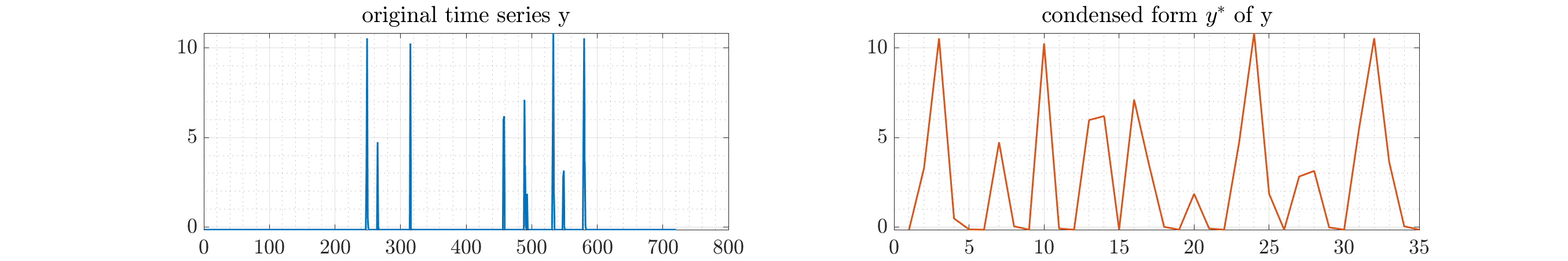}
\caption{The left column shows two time series $x$ and $y$ of length $720$ from the UCR dataset \emph{SmallKitchenAppliances}. The right column shows the corresponding condensed forms $x^*$ and $y^*$ obtained by collapsing constant segments to singletons. The dtw-distances are $\delta(x, y) = 73.2$ and $\delta(x^*,y^*) = 49.1$ suggesting that the dtw-distance is not invariant under condensation, which we regard as a special form of warping.}
\label{fig:ex_condensation}
\end{figure}

Even though Fig.~\ref{fig:ex_condensation} refers to one of the most extreme examples one can find in the UCR archive, sacrificing potential space-saving ratios and speed-up factors generally ask for a justification. An example of such a justification would be a better solution quality in time series mining tasks. Evidence to justify the inconsistency between the idea behind the dtw-distance and how this distance has been actually applied is missing.

We frame the inconsistency-problem into the wider perspective of the warping-invariance problem. Though the dtw-distance has been proposed to account for temporal variation potential effects caused by the lack of warping-invariance have not been clearly disclosed---leaving a number of open questions:
\begin{enumerate}
\itemsep0em
\item 
What are the negative effects contributed by the lack of warping-invariance?
\item 
How can we convert the dtw-distance to a warping-invariant distance? 
\item
What are the practical implications of a warping-invariant version of the dtw-distance?
\end{enumerate}
Possible adverse effects contributed by the absence of warping-invariance would ask for a solution. In this case, a possible solution would aim at converting the dtw-distance to a warping-invariant distance. This in turn raises the question, whether such a solution is not only theoretically fancier but also has practical benefits. 

To close the gap of the warping-invariance problem, we convert the dtw-distance to a warping-invariant semi-metric, called \emph{time-warp-invariant} (twi) distance. The resulting twi-distance corresponds to the dtw-distance on the subset of condensed forms. The construction of the twi-distance answers the second question and forms the foundation to answer the first and third question. 

To answer the first question, we show that the lack of warping-invariance contributes to an adverse effect of the nearest-neighbor rule in dtw-spaces. The nearest-neighbor rule is an important subroutine for several time series mining methods. Examples include similarity search, $k$-nearest-neighbor classification \cite{Abanda2018,Bagnall2017}, k-means clustering \cite{Rabiner1979,Hautamaki2008}, learning vector quantization \cite{Jain2018,Somervuo1999}, and self-organizing maps \cite{Kohonen1998}. The peculiarity of the nearest-neighbor rule affects all the time series mining methods just mentioned.  As an example, we show that it is possible to arbitrarily increase the cluster quality of k-means without changing the clusters under certain conditions. 

Finally, we answer the third question on the practical benefit of the twi-distance with respect to (i) computation time, (ii) storage consumption, and (iii) solution quality. The results show that the error rates of the twi and dtw nearest-neighbor are practically equivalent in a Bayesian sense, but the twi-distance consumes less storage and is faster to compute than the dtw-distance for a broad range of problems. 

\medskip

The rest of this article is structured as follows: Section 2 provides an informal overview of the main theoretical results. In Section 3, we point to peculiarities of the dtw-distance. Experiments are presented and discussed in Section 4. Finally, Section 5 concludes with a summary of the main results, their impact, and an outlook to further research. The theoretical treatment including all proof has been delegated to the appendix.

\section{A Warping-Invariant DTW Distance}\label{sec:results}

This section introduces the time-warp-invariant (twi) distance and shows that this distance is a warping-invariant semi-metric induced by the dtw-distance. We first introduce metric properties and the dtw-distance. Then we define the concept of warping-invariance. Then we present the main results to convert a dtw-space into a warping-invariant semi-metric space. Finally, we discuss some computational issues of the twi-distance. The theoretical treatment is delegated to \ref{sec:theory}. 

\subsection{Metric Properties}

In this section, we introduce metric properties and define the notion of semi-metric and pseudo-metric. A metric on a set $\S{X}$ is a function $d:\S{X} \times \S{X} \rightarrow \R$ satisfying the following properties:
\begin{enumerate}
\itemsep0em
\item 
$d(x, y) \geq 0$ \hfill (non-negativitiy)
\item 
$d(x, y) = 0 \;\Leftrightarrow \;x = y$ \hfill (identity of indiscernibles)
\item
$d(x,y) = d(y,x)$ \hfill (symmetry)
\item
$d(x,z) \leq d(x,y) + d(y,x)$ \hfill (triangle inequality)
\end{enumerate}
for all $x, y, z \in \S{X}$. We call $d$ a distance function. A semi-metric is a distance function satisfying the first three metric properties $(1)-(3)$. A pseudo-metric is a distance function satisfying the metric properties $(1), (3),$ and $(4)$.

\subsection{The DTW-Distance}
For a given $n \in \N$, we write $[n] = \cbrace{1, \ldots, n}$. A \emph{time series} of length $n$ is a sequence $x = (x_1, \ldots, x_n)$ with elements $x_i \in \R$ for all $i \in [n]$. We denote the set of all time series of finite length by $\S{T}$. 
Consider the ($m \times n$)-grid defined by
\[
[m] \times [n] = \cbrace{(i,j) \,:\, i \in [m], j \in [n]}.
\]
A \emph{warping path} of order $m \times n$ and length $\ell$ is a sequence $p = (p_1 , \dots, p_\ell)$ through the grid $[m] \times [n]$ consisting of $\ell$ points $p_l = (i_l,j_l) \in [m] \times [n]$ such that
\begin{enumerate}
\item $p_1 = (1,1)$ and $p_\ell = (m,n)$ \hfill (\emph{boundary conditions})
\item $p_{l+1} - p_{l} \in \cbrace{(1,0), (0,1), (1,1)}$ for all $l \in [\ell-1]$ \hfill(\emph{step condition})
\end{enumerate}
We denote the set of all warping paths of order $m \times n$ by $\S{P}_{m,n}$. Figure \ref{fig:warping-paths}(a) depicts examples of warping paths through a ($3 \times 3$)-grid.

\begin{figure}[t]
\centering
\includegraphics[width=0.95\textwidth]{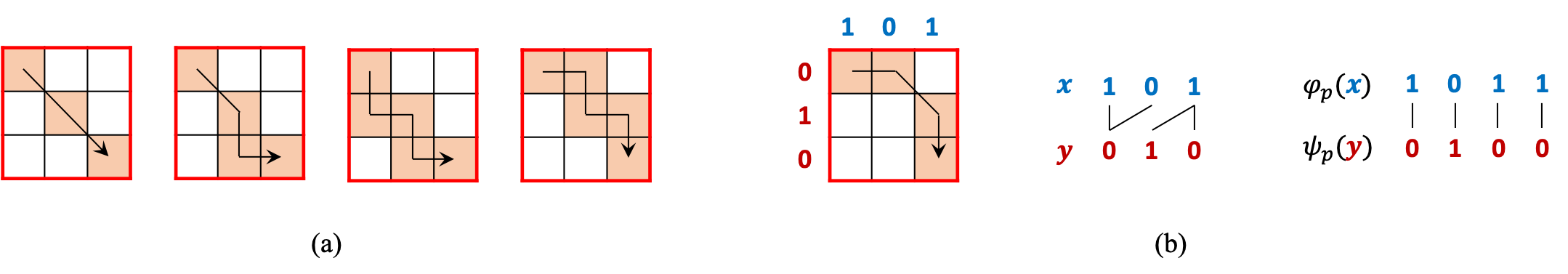}
\caption{Figure (a) shows examples of warping paths through a ($3 \times 3$)-grid. Figure (b) shows another warping path $p$ through a ($3 \times 3$)-grid, the corresponding warping of two time series $x$ and $y$, and the expansions $\phi_p(x)$ and $\psi_p(y)$. The lines between elements of $x$ and $y$ ($\phi_p(x)$ and $\psi_p(y)$) show the correspondences that contribute to the cost $C_p$ of warping $x$ and $y$ along $p$.} 
\label{fig:warping-paths}
\end{figure}

Suppose that $p = (p_1, \ldots, p_\ell) \in \S{P}_{m,n}$ is a warping path with points $p_l = (i_l, j_l)$ for all $l \in [\ell]$. Then $p$ defines an  (warping) of time series $x = (x_1, \ldots, x_m)$ and $y = (y_1, \ldots, y_n)$ to time series $\phi_p(x) = (x_{i_1}, \ldots, x_{i_\ell})$ and $\psi_p(y) = (y_{j_1}, \ldots, y_{j_\ell})$ of the same length $\ell$. By definition, the length $\ell$ of a warping path satisfies $\max(m,n) \leq \ell < m+n$. This shows that $\ell \geq \max(m,n)$ and therefore $\phi_p(x)$ and $\psi_p(y)$ are indeed expansions of $x$ and $y$. Figure \ref{fig:warping-paths}(b) illustrates the concept of expansion.

The \emph{cost} of warping time series $x$ and $y$ along warping path $p$ is defined by
\begin{equation*}
C_p(x, y) = \normS{\phi_p(x)-\psi_p(y)}{^2} = \sum_{(i,j) \in p} \argsS{x_i-y_j}{^2},
\end{equation*}
where $\norm{\cdot}$ denotes the Euclidean norm and $\phi_p$ and $\psi_p$ are the expansions defined by $p$. Then the \emph{dtw-distance} of $x$ and $y$ is of the form
\begin{align*}
\dtw(x, y) = \min \cbrace{\sqrt{C_p(x, y)} \,:\, p \in \S{P}_{m,n}}.
\end{align*}
A warping path $p$ with $C_p(x, y) = \dtw^2(x, y)$ is called an \emph{optimal warping path} of $x$ and $y$. By definition, the dtw-distance minimizes the Euclidean distance between all possible expansions that can be derived from warping paths. Computing the dtw-distance and deriving an optimal warping path is usually solved by applying a dynamic program \cite{Sakoe1978,Vintsyuk1968}.

The dtw-distance satisfies only the metric properties $(1)$ and $(3)$. Thus, the dtw-distance is neither a metric, nor a semi- and pseudo-metric.

\subsection{Warping-Invariance}\label{subsec:main-warping-invariance}

In this section, we define the concept of warping-invariance. For this, we introduce the notions of expansion and compression. Expansions have been introduced in the previous section in the context of warping. Here, we introduce expansions and compressions as operations on time series. 

\medskip

An expansion of a time series is obtained by replicating a subset of its elements. More precisely, suppose that $x =(x_1, \ldots, x_n)$ is a time series. A time series $x'$ is an \emph{expansion} of time series $x$, written as $x' \succ x$, if there are indexes $i_1, \ldots, i_k \in [n]$ and positive integers $r_1, \ldots, r_k \in \N$ for some $k \in \N_0$ such that $x'$ is of the form
\[
x' = (x_1, \ldots, 
\underbrace{x_{i_1}, \ldots, x_{i_1}}_{r_1\text{\,-times}}, x_{i_1 + 1}, \ldots, 
\underbrace{x_{i_2}, \ldots, x_{i_2}}_{r_2\text{\,-times}}, x_{i_2 + 1}, \ldots, 
\underbrace{x_{i_k}, \ldots, x_{i_k}}_{r_k\text{\,-times}}, x_{i_k+1}\ldots, x_n)
\]
By setting $k = 0$ we find that a time series is always an expansion of itself. A time series $x$ is a \emph{compression} of time series $x'$, written as $x \prec x'$, if $x'$ is an expansion of $x$. A time series $z$ is a \emph{common compression} of time series $x$ and $y$ if $z$ is a compression of $x$ and of $y$, that is $z \prec x$ and $z \prec y$. We write $x \sim y$ if both time series have a common compression. Figure \ref{fig:expansions} depicts the relationships between time series that have a common compression in form of a directed rooted tree of infinite depth. The root is the common compression of all nodes in the tree. 

\begin{figure}[t]
\centering
\includegraphics[width=0.8\textwidth]{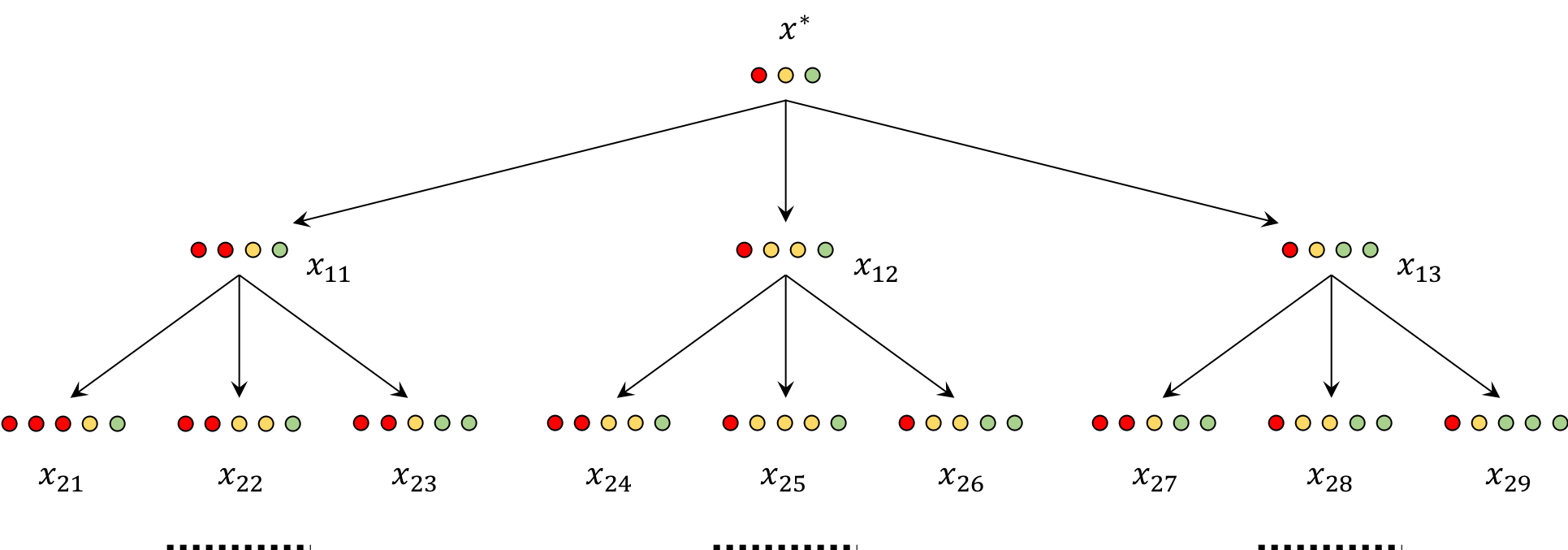}
\caption{Directed rooted tree of infinite depth with root $x^*$. Sequences with filled balls represent time series. The balls refer to the elements of the corresponding time series and different colors of the balls refer to different values. Tree nodes represent time series and directed edges represent one-element expansions from the parent to the child node. The parents are the one-element compressions of their children. If two nodes $x$ and $y$ have a common ancestor $z$, then $z$ is a common compression of $x$ and $y$. The root $x^*$ is a common compression of all nodes in the tree.}
\label{fig:expansions}
\end{figure}

A distance function $d:\S{T} \times \S{T} \rightarrow \R_{\geq 0}$ is \emph{warping-invariant} if 
\[
d(x, y) = d(x', y')
\]
for all time series $x, y, x', y' \in \S{T}$ with $x \sim x'$ and $y \sim y'$. Observe that warping-invariance as defined here is more general than the notion of invariance introduced in Section \ref{sec:intro}. The difference is that invariance in Section \ref{sec:intro} assumes that both time series undergo the same transformation, whereas in this section the transformation can be different for both time series. The next example shows that the dtw-distance is not warping-invariant.
\begin{example}\label{ex:warping-invariance} Let $x = (0, 1)$ and $x' = (0, 1, 1)$ be two time series. Then $x'$ is an expansion of $x$ and $\dtw(x, x') = 0$. Suppose that $y = (0, 2)$ is another time series. Then we have
\begin{align*}
\dtw(x, y)^2 &= (0-0)^2+(1-2)^2 = 1\\
\dtw(x', y)^2 &= (0-0)^2+(1-2)^2+(1-2)^2 = 2.
\end{align*}
This implies $\dtw(x, y) \neq \dtw(x', y)$, where $x$ is a common compression of $x, x'$ and $y$ is a compression of itself. Hence, the dtw-distance is not warping-invariant. 
\QED
\end{example}

Warping-invariance as introduced here must not be confused with the notion of warping-invariance and related terms in the literature  \cite{Batista2014,Chavoshi2016,Chen2013,Mueen2018,Silva2016}. In \cite{Jain2015}, I falsely insinuated that the dtw-distance is warping-invariant. There are also statements closely related to the notion of warping-invariance that are prone to misunderstandings. For example:
\begin{itemize}
\itemsep0em
\item \emph{The sequences are \enquote{warped} non-linearly in the time dimension to determine a measure of their similarity independent of certain non-linear variations in the time dimension} \cite{Khosrow-Pour2015}. 
\item \emph{Dynamic time warping (DTW) was introduced} [...] \emph{to overcome this limitation} [of the Euclidean distance] \emph{and give intuitive distance measurements between time series by ignoring both global and local shifts in the time dimension} \cite{Salvador2007}. 
\end{itemize}
As Example \ref{ex:warping-invariance} indicates, the dtw-distance is neither independent of non-linear variations nor ignores global and local shifts in the time dimension.

\subsection{The Time-Warp-Invariant Distance}\label{subsec:results:approach}

In this section, we present the main results of the theoretical exposition. In particular, we introduce the time-warp-invariant (twi) distance and show that the twi-distance is a warping-invariant semi-metric. \ref{sec:theory} contains the full theoretical treatment including all proofs. 

\medskip

The blueprint to convert a dtw-space into a semi-metric space follows the standard approach to convert a pseudo-metric space into a metric space. However, additional issues need to be resolved, because the dtw-distance satisfies less metric properties than a pseudo-metric. According to the blueprint, we proceed as follows: First, we introduce an equivalence relation $\sim$ on time series. Then we construct the quotient space of $\sim$. Next, we endow the quotient space with a quotient distance, called twi-distance. Finally, we show that the twi-distance is a warping-invariant semi-metric.

\medskip

We begin with introducing an equivalence relation on time series. Two time series are said to be \emph{warping-identifiable} if their dtw-distance is zero. Formally, \emph{warping identification} is the relation defined by $x \sim y \,\Leftrightarrow\, \dtw(x, y) = 0$ for all time series $x, y \in \S{T}$. From \ref{sec:theory} follows that warping identification and common compression are equivalent definitions of the same relation. In light of Fig.~\ref{fig:expansions}, this suggests that all pairs of time series from the same tree are warping-identifiable. The next result shows that warping identification is indeed an equivalent relation and thereby forms the foundation for converting the dtw-distance to a semi-metric.

\begin{proposition}\label{prop:warping-identification-class}
Warping identification $\sim$ is an equivalence relation on $\S{T}$. 
\end{proposition}

\medskip

The assertion of Prop.~\ref{prop:warping-identification-class} is not self-evident, because warping paths are not closed under compositions and the dtw-distance fails to satisfy the triangle inequality. Therefore, it might be the case that there are time series $x, y,z$ such that $\delta(x, y) = \delta(y, z) = 0$ but $\delta(x, z) > \delta(x, y) + \delta(y, z) = 0$. 

\begin{figure}[t]
\centering
\includegraphics[width=0.55\textwidth]{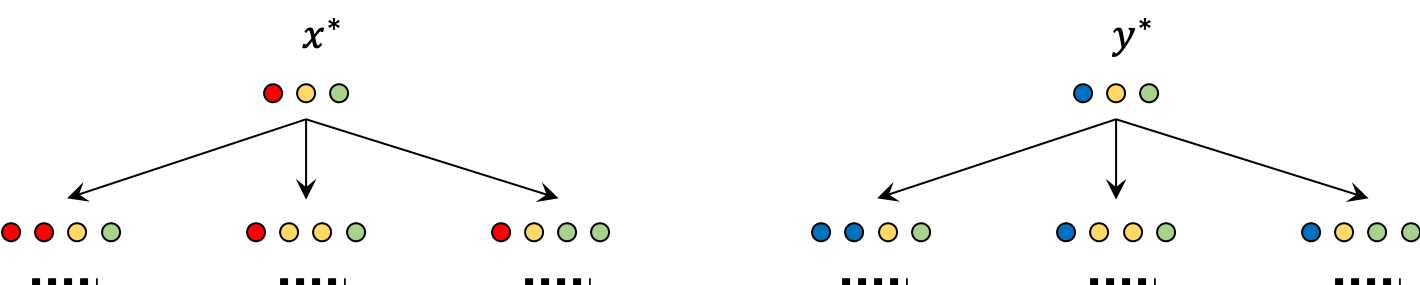}
\caption{Visualization of two equivalence classes by trees. The roots $x^*$ and $y^*$ are time series that cannot be further compressed.}
\label{fig:equivalence_classes}
\end{figure}

\medskip

We continue with constructing a new set, the so-called quotient space of equivalence relation $\sim$. This task consists of two steps. The first step forms equivalence classes and the second step collects all equivalence classes to a new set. This new set is the desired quotient space. 

The equivalence class of a time series $x$ is the set 
\[
[x] = \cbrace{y \in \S{T}\,:\, y \sim x}
\] 
consisting of all time series that are warping-identifiable to $x$. Time series from the same equivalence class have zero dtw-distance and time series from different equivalence classes have non-zero dtw-distance. Figure \ref{fig:equivalence_classes} depicts examples of two different equivalence classes. 

The set 
\[
\S{T}^* = \cbrace{[x] \,:\ x \in \S{T}}
\]
of all equivalence classes is the quotient space of $\S{T}$ by warping identification $\sim$. With regard to the visual representation of Fig.~\ref{fig:equivalence_classes}, the quotient space consists of all trees rooted at non-compressible time series. 

\medskip

Next, we endow the quotient space $\S{T}^*$ with a distance. The twi-distance is defined by
\begin{align}\label{eq:twi-distance-01}
\delta^*: \S{T}^* \times \S{T}^* \rightarrow \R_{\geq 0}, \quad ([x], [y]) \mapsto \inf_{x' \in [x]} \; \inf_{y' \in [y]} \; \delta(x', y').
\end{align}

\medskip

Readers unfamiliar with the concept of infimum may safely replace it by the minimum in this particular case. The twi-distance is a dtw-distance between special sets, the equivalence classes under warping identification. As a distance on sets, the twi-distance between two equivalence classes is the infimum dtw-distance between their corresponding representatives. For example, the twi-distance between the equivalence classes represented by the trees in Fig.~\ref{fig:equivalence_classes} is the infimum of all dtw-distances between time series obtained by expanding $x^*$ and time series obtained by expanding $y^*$.

\medskip

The last step consists in showing that the twi-distance is a warping-invariant semi-metric. However, the twi-distance 
as defined in Eq.~\eqref{eq:twi-distance-01} introduces an  additional problem. The naive way to compute the twi-distance by comparing infinitely many dtw-distances between time series from either equivalence class is computationally unfeasible.

To show warping-invariance, semi-metric properties, and to overcome the problem of an incomputable form, we derive an equivalent formulation of the twi-distance. For this, we introduce irreducible time series. A time series is said to be \emph{irreducible} if it cannot be expressed as an expansion of a shorter time series. For example, the roots of the trees in Fig.~\ref{fig:expansions} and \ref{fig:equivalence_classes} are all irreducible time series. The next result shows that an equivalence class can be generated by expanding an irreducible time series in all possible ways. 
\begin{proposition}\label{prop:generator-of-[x]}
For every time series $x$ there is an irreducible time series $x^*$ such that
\begin{align*}
[x] &= \cbrace{y \in \S{T} \,:\, y \succ x^*}.
\end{align*}
\end{proposition}

\medskip

The irreducible time series $x^*$ in Prop.~\ref{prop:generator-of-[x]} is called the \emph{condensed form} of $x$. Every time series has a unique condensed form by \ref{sec:theory}, Prop.~\ref{prop:existence-of-condensed-form}. Hence, $x^*$ is the condensed form of all time series contained in the equivalence class $[x]$. Proposition~\ref{prop:generator-of-[x]} states that the equivalence classes are the trees rooted at a their respective condensed forms as illustrated in Figure \ref{fig:expansions} and \ref{fig:equivalence_classes}. Next, we show that expansions do not decrease the dtw-distance to other time series. 

\begin{proposition}\label{prop:expansion-inequality}
Let $x, x' \in \S{T}$ such that $x' \succ x$. Then $\dtw(x', z) \geq \dtw(x,z)$ for all $z \in \S{T}$.
\end{proposition}

\medskip

A special case of Prop.~\ref{prop:expansion-inequality} has been proved in \cite{Brill2019}. We will use Prop.~\ref{prop:expansion-inequality} in Section \ref{sec:peculiarities} to discuss peculiarities of data-mining applications on time series. By invoking Prop.~\ref{prop:generator-of-[x]} and Prop.~\ref{prop:expansion-inequality} we obtain the first key result of this contribution.

\begin{theorem}\label{theorem:semi-metric}
The twi-distance $\dtw^*$ induced by the dtw-distance $\dtw$ is a well-defined semi-metric satisfying 
\[
\dtw^*([x], [y]) = \dtw(x^*,y^*)
\]
for all $x, y \in \S{T}$ with condensed forms $x^*$ and $y^*$, respectively. 
\end{theorem} 

\medskip

Theorem \ref{theorem:semi-metric} converts the dtw-distance to a semi-metric and shows that the twi-distance can be efficiently computed as discussed in Section \ref{subsec:compIssues}.

The second key result of this contribution shows that the semi-metric induced by the dtw-distance is warping-invariant. Since warping-invariance is defined on the set $\S{T}$ rather than on the quotient space $\S{T}^*$, we extend the twi-distance $\dtw^*$ to a distance on $\S{T}$ by virtue of
\[
\dtw^\sim: \S{T} \times \S{T} \rightarrow \R_{\geq 0}, \quad (x, y) \mapsto \dtw^\sim(x, y) = \dtw^*([x], [y]).
\]
We call $\dtw^\sim$ the \emph{canonical extension} of $\dtw^*$.

\begin{theorem}\label{theorem:warping-invariance} 
The canonical extension $\dtw^\sim$ of the twi-distance $\dtw^*$ is warping-invariant.
\end{theorem}

\subsection{Computational Issues}\label{subsec:compIssues}
In this section, we discuss computational issues of the twi-distance and point to optimization techniques in nearest-neighbor search. 

\paragraph*{Computing the TWI-Distance.}

\begin{figure}
\centering
\includegraphics[width=0.8\textwidth]{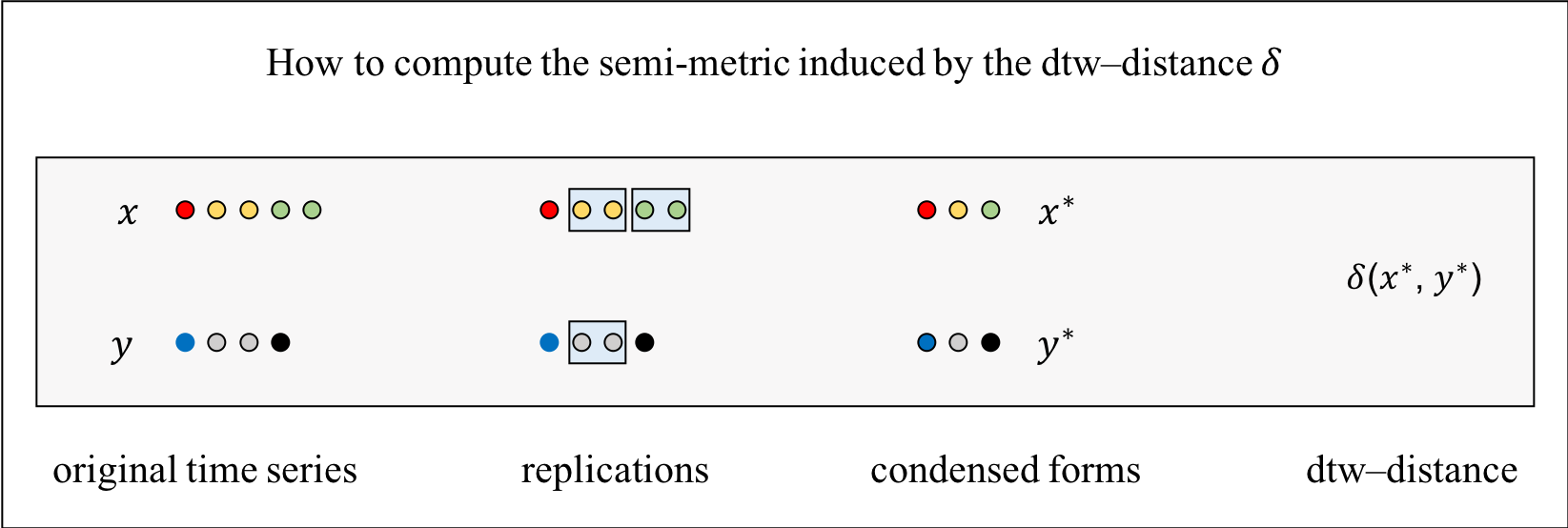}
\caption{Process of computing the proposed semi-metric when the time series to be compared are not irreducible. The first column shows two time series $x$ and $y$ consisting of $5$ and $4$ elements, respectively. Elements are shown by filled balls, where different colors refer to different real values. The second column identifies consecutive replications within the time series as highlighted by the blue-shaded boxes. The third column collapses replications to singletons resulting in condensed forms $x^*$ and $y^*$. The semi-metric is the dtw-distance between the condensed forms $x^*$ and $y^*$.}
\label{fig:semi-metric}
\end{figure}

The twi-distance and dtw-distance only coincide on irreducible time series. If at least one of the two time series to be compared is not irreducible, there are two ways to compute the twi-distance: 
\begin{enumerate}
\itemsep0em
\item The traditional way assumes the existence of time series that are not irreducible. Then the twi-distance $\dtw^*([x], [y])$ can be computed by first compressing $x$ and $y$ to their condensed forms $x^*$ and $y^*$, resp., in linear time and then taking their dtw-distance $\dtw(x^*,y^*)$. 
\item From the perspective of applications that demand warping-invariance, there is no point of storing all but one element of a constant sub-sequence. Instead, one assumes that all time series are irreducible and then computes the twi-distance as an ordinary dtw-distance on irreducible time series.
\end{enumerate}
Figure \ref{fig:semi-metric} illustrates how to compute the twi-distance for the case that both time series to be compared are not irreducible. Let $\abs{x}$ denote the length of a time series $x$. Suppose that $x$ and $y$ are two time with corresponding condensed forms $x^*$ and $y^*$, respectively. Then  $\rho_{\text{c}}(x) = \abs{x}/\abs{x^*}$ is the compression-ratio of $x$. The higher the compression-ratio of $x$, the shorter is its condensed form $x^*$. The speed-up factor of the twi-distance over the dtw-distance is given by
\begin{align*}
\phi_1 = \frac{\abs{x}\abs{y}}{\abs{x^*}\abs{y^*} + \abs{x^*} + \abs{y^*}}  
\quad  \text{and} \quad
\phi_2 = \rho_{\text{c}}(x) \cdot \rho_{\text{c}}(y) = \frac{\abs{x}\abs{y}}{\abs{x^*}\abs{y^*}},
\end{align*}
where $\phi_1$ assumes that time series are not necessarily irreducible and $\phi_2$ assumes that all time series are irreducible. Suppose that $n = \abs{x} = \abs{y}$ and $m = \max\cbrace{\abs{x^*}, \abs{y^*}}$. Then the first speed-up factor $\phi_1$ satisfies
\[
\phi_1 \geq \frac{n^2}{m^2 + 2m} > \frac{n^2}{(m+1)} \geq 1
\]
if $n > m$. Thus, computing the twi-distance is more efficient than computing the dtw-distance once there is at least one replication in both time series $x$ and $y$. In particular, the twi-distance is much more efficient to compute than the dtw-distance for extremely long time series with high compression-ratio. Such time series occur, for example, when measuring electrical power of household appliances with a sampling rate of a few seconds collected over several months, twitter activity data sampled in milliseconds, and human activities inferred from a smart home environment \cite{Mueen2018}. For such time series, storage consumption can be reduced to over $90\%$.

\medskip

Recently, two algorithms have been devised to cope with time series with high compression-ratio. The first algorithm is AWarp proposed by \cite{Mueen2018}. This algorithm exactly computes the dtw-distance for time series consisting of binary values and is an approximation of the dtw-distance otherwise. The second algorithm is the bdtw-distance (block-dtw) proposed by \cite{Sharabiani2018}. The bdtw-algorithm generalizes AWarp by exactly computing the dtw-distance not only for binary but also for any two-valued time series.  For time series with more than two-values, the bdtw-distance is an approximation of the dtw-distance and comes with a lower and upper bound. 

Both algorithms have a similar speed-up factor as the twi-distance but are slower by a small factor. This factor is caused by additional operations such as multiplying every element with the number of its repetitions in order to approximate the dtw-distance.

\paragraph*{Optimizing Nearest-Neighbor Search.}

One common task in time series mining is nearest-neighbor search \cite{Silva2018}. Given a dataset $\S{D}$ consisting of $N$ time series and a query time series $q$, the goal is to find a time series $x \in \S{D}$ that is closest to $q$ with respect to the dtw-distance. A naive implementation compares the query to all $N$ time series to identify the nearest-neighbor. The naive approach is unfeasible for large $N$ and for long time series due to the quadratic complexity of the dtw-distance. 

As a consequence, nearest-neighbor search has been continually optimized by devising combinations of sophisticated pruning methods such as lower bounding and early abandoning  \cite{Silva2018}. Since the twi-distance is a dtw-distance on irreducible time series, all pruning methods applicable to the dtw-distance can also be applied to the twi-distance. In some cases, we need to ensure that the irreducible time series are of the same length. This holds for lower-bounding techniques such as LB\_Lemire \cite{Lemire2009}and LB\_Keogh \cite{Keogh2002}. We can align irreducible time series to identical length using any technique that preserves irreducibility. An example of an irreducibility-preserving technique is linear interpolation.

\section{Peculiarities of the DTW-Distance}\label{sec:peculiarities}

In this section, we discuss some peculiarities of the dtw-distance, which do not occur for the twi-distance. 

\subsection{The Effect of Expansions on the Nearest-Neighbor Rule}

Suppose that $x'$ is an expansion of $x$. Then Prop.~\ref{prop:expansion-inequality} yields
\begin{align}
\label{eq:warping-identification}
\dtw(x, x') &= 0\\
\label{eq:expansion-inequality}
\dtw(x, y)\phantom{'} &\leq \dtw(x', y)
\end{align}
for all time series $y$. Equation \eqref{eq:warping-identification} states that a time series is warping identical with its expansions. Equation \eqref{eq:expansion-inequality} says that expansions do not decrease the dtw-distance to other time series. Equation \eqref{eq:expansion-inequality} affects the nearest-neighbor rule. To see this, we consider a set $\S{D} = \cbrace{x, y} \subseteq \S{T}$ of two prototypes. The Voronoi cells of $\S{D}$ are defined by 
\begin{align*}
\S{V_D}(x) &= \cbrace{z \in \S{T} \,:\, \dtw(x, z) \leq \dtw(y, z)} \qquad \text{and} \qquad
\S{V_D}(y) = \cbrace{z \in \S{T} \,:\, \dtw(y, z) \leq \dtw(x, z)}.
\end{align*}
Suppose that $z$ is a time series. If $z \in \S{V_D}(x)$ then prototype $x$ is the nearest-neighbor of $z$. Otherwise, prototype $y$ is the nearest-neighbor of $z$. To break ties, we arbitrarily assign $x$ as nearest-neighbor for all time series from the boundary $\S{V_D}(x) \cap \S{V_D}(y)$.
The nearest-neighbor rule assigns $z$ to its nearest-neighbor. 

Suppose that we replace prototype $x$ by an expansion $x'$ to obtain a modified set $\S{D}' = \cbrace{x', y}$ of prototypes. Although $x'$ and $x$ are warping identical by Eq.~\eqref{eq:warping-identification}, the Voronoi cells of $\S{D}$ and $\S{D}'$ differ. From Eq.~\eqref{eq:expansion-inequality} follows that $\S{V_{D'}}(x') \subseteq \S{V_D}(x)$, which in turn implies $\S{V_{D'}}(y) \supseteq \S{V_D}(y)$. Keeping prototype $y$ fixed, expansion of prototype $x$ decreases its Voronoi cell and thereby increases the Voronoi cell of $y$. This shows that the nearest-neighbor rule in dtw-spaces depends on temporal variation of the prototypes. 

This peculiarity of the nearest-neighbor rule affects time series mining methods in dtw-spaces such as similarity search, $k$-nearest-neighbor classification, k-means clustering, learning vector quantization and self-organizing maps.

\subsection{The Effect of Expansions on k-Means}

\renewcommand{\figurename}{Box}
\begin{figure}
\centering
\colorbox{gr}{\textcolor{black}{
\begin{minipage}{\textwidth}
\footnotesize
\textbf{A Brief Review on k-Means in DTW-Spaces}
\\[0.5em]
\phantom{\ \ \ } 
The foundations of current state-of-the-art k-means algorithms in dtw-spaces have been established in the 1970ies mainly by Rabiner and his co-workers  with speech recognition as the prime application \cite{Rabiner1979,Wilpon1985}. The early foundations fell largely into oblivion and where successively rediscovered, consolidated, and improved in a first step by \cite{Abdulla2003} in 2003 and then finalized by \cite{Hautamaki2008} in 2008. The k-means algorithm proposed by  \cite{Hautamaki2008} has been accidentally misrepresented in \cite{Petitjean2011} as a k-medoid algorithm. In the same work, \cite{Petitjean2011} proposed and explored the popular DBA algorithm for mean computation and plugged this algorithm into k-means \cite{Petitjean2011,Petitjean2016}. The approach to time series averaging proposed by \cite{Hautamaki2008} and the DBA algorithm proposed by \cite{Petitjean2011}  are both majorize-minimize algorithms. Conceptual differences between both algorithms are unclear, in particular after inspecting both implementations. Thereafter, the k-means algorithm in dtw-spaces has been generalized in different directions: to  weighted and kernel time warp measures \cite{Soheily-Khah2016}, to constrained dtw-distances \cite{Morel2018},  and to smoothed k-means by using a soft version of the dtw-distance \cite{Cuturi2017}.
\\[-0.7em]
\end{minipage}}}
\vspace{-0.3cm}
\caption{}
\label{box:k-means}
\end{figure}
\renewcommand{\figurename}{Figure}

As an example, we discuss how the peculiarities of the nearest-neighbor rule affect the k-means algorithm. Box \ref{box:k-means} provides a brief review about k-means in dtw-spaces. 

\medskip

To apply k-means in dtw-spaces, we need a concept of average of time series. Different forms of time series averages have been proposed (see \cite{Schultz2018} for an overview). One principled formulation of an average is based on the notion of Fr\'echet function \cite{Frechet1948}: Suppose that $\S{S} = \args{x_1, \dots, x_n}$ is a sample of $n$ time series $x_i \in \S{T}$. Then the \emph{Fr\'echet function} of $\S{S}$ is defined by
\[
F: \S{T} \rightarrow \R, \quad z \mapsto \sum_{i=1}^n \dtw\!\argsS{x_i, z}{^2},
\]
A \emph{sample mean} of $\S{S}$ is any time series $\mu \in \S{T}$ that satisfies $F(\mu) \leq F(z)$ for all $z \in \S{T}$. A sample mean exists but is not unique in general \cite{Jain2016}. Computing a mean of a sample of time series is NP-hard \cite{Bulteau2018}. Efficient heuristics to approximate a mean of a fixed and pre-specified length are stochastic subgradient methods \cite{Schultz2018} and majorize-minimize algorithms \cite{Hautamaki2008,Petitjean2011}.

\medskip

Next, we generalize the k-means algorithm to dtw-spaces by replacing arithmetic means of vectors with sample means of time series. For this, we first introduce some notations. A partition of a set $\S{S}= \cbrace{x_1, \dots, x_n} \subseteq \S{T}$ of time series is a set $\S{C} = \args{\S{C}_1, \ldots, \S{C}_k}$ of $k$ non-empty subsets $\S{C}_i \subseteq \S{S}$, called \emph{clusters}, such that $\S{S}$ is the disjoint union of these clusters. By $\Pi_k$ we denote the set of all partitions consisting of $k$ clusters. The objective of k-means is to minimize the cost function
\[
J: \Pi_k \rightarrow \R, \quad \args{\S{C}_1, \ldots, \S{C}_k} \mapsto \sum_{i=1}^k \sum_{x \in \S{C}_i} \dtw(x, \mu_i)^2,
\]
where $x \in \S{C}_i$ if the centroid $\mu_i$ of cluster $\S{C}_i$ is the nearest-neighbor of $x$ for all $i \in [k]$. By definition of the cost function $J$, the centroid $\mu_i$ of cluster $\S{C}_i$ coincides with the sample mean of $\S{C}_i$. Thus, we can equivalently express the cost function $J$ by
\[
J \args{\S{C}_1, \ldots, \S{C}_k} = \sum_{i=1}^k F_i(\mu_i)
\]
where $F_i$ is the Fr\'echet function of cluster $\S{C}_i$. To minimize the cost function $J$, the standard k-means algorithm starts with an initial set $\mu_1, \ldots, \mu_k$ of centroids and then proceeds by alternating between two steps until termination:
\begin{enumerate}
\itemsep0em
\item \emph{Assignment step}: Assign each sample time series to the cluster of its closest centroid. 
\item \emph{Update step}: Recompute the centroids for every cluster. 
\end{enumerate}
Due to non-uniqueness of sample means, the performance of k-means does not only depend on the choice of initial centroids as in Euclidean spaces, but also on the choice of centroids in the update step. In the following, we present two examples that illustrate a peculiar behavior of k-means in dtw-spaces. 

\begin{example}\label{example:assignment-k-means}
We assume that the four time series 
\begin{align*}
x_1 &= (-1, 0, 0) & x_3 &= (0, 2, 3)\\
x_2 &= (-1, 0, 2) & x_4 &= (1, 2, 3)
\end{align*}
are partitioned into two clusters $\S{C}_1 = \cbrace{x_1, x_2}$ and $\S{C}_2 = \cbrace{x_3, x_4}$. Cluster $\S{C}_1$ has a unique mean $\mu_1 = (-1,0,1)$. Cluster $\S{C}_2$ has infinitely many means as indicated by Figure \ref{fig:non-condensed}. For every $r \in \N$ the time series 
\[
\mu_2^r = (0.5, 2, \underbrace{3, \ldots, 3}_{r-\text{times}}),
\]
is a mean of $\S{C}_2$. From Eq.~\eqref{eq:warping-identification} and the transitivity of expansions follows that $\mu_2^r$ and $\mu_2^s$ are warping identical for all $r, s \in \N$. Equation \eqref{eq:expansion-inequality} yields $\dtw(\mu_2^1, y) \leq \cdots \leq \dtw(\mu_2^r, y)$ for all time series $y \in \S{T}$ and all $r \in \N$. Hence, with increasing number $r$ of replications, the Voronoi cell of site $\mu_2^r$ decreases, whereas the Voronoi cell of site $\mu_1$ increases. 
\QED
\end{example}

\medskip

Example \ref{example:assignment-k-means} shows that the Voronoi cell of a centroid $\mu_i$ can be externally controlled by expanding or compressing $\mu_i$ without changing the Fr\'echet variation $F(\mu_i)$. This, in turn, affects the assignment step, which is based on the nearest-neighbor rule. 

\begin{figure*}[t]
\centering
 \includegraphics[width=0.85\textwidth]{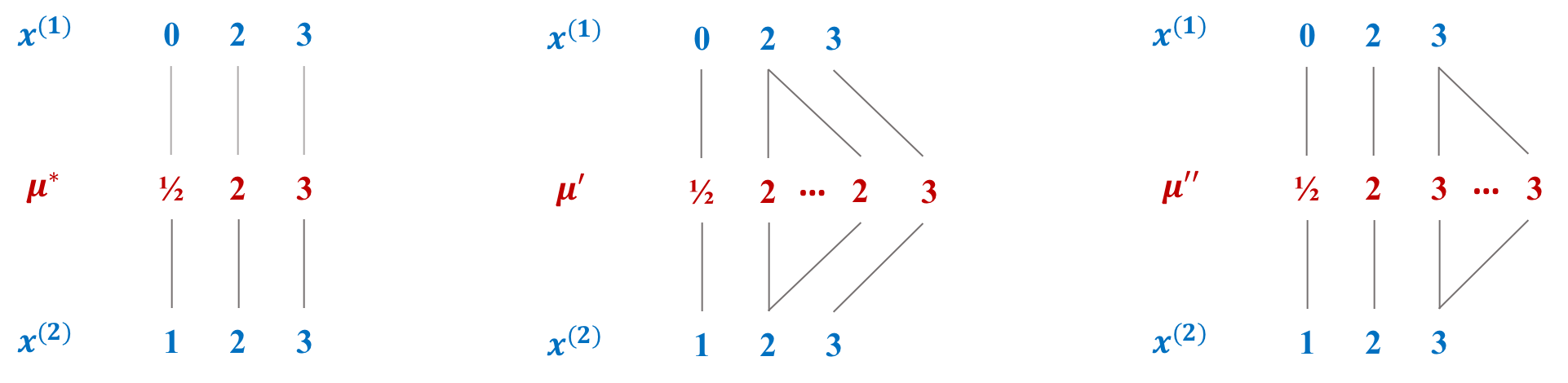}
 \caption{Means $\mu^*, \mu', \mu''$ of sample $\S{S} = \args{x^{(1)}, x^{(2)}}$. The means $\mu'$ and $\mu''$ are expansions of $\mu^*$.}
\label{fig:non-condensed}
\end{figure*}

\begin{example}\label{example:non-condensed-k-means}
Consider the two clusters of the four time series from Example \ref{example:assignment-k-means}. One way to measure the quality of the clustering $\S{C} = \cbrace{\S{C}_1, \S{C}_2}$ is by \emph{cluster cohesion} and \emph{cluster separation}. Cluster cohesion is typically defined by
\[
F_{\text{cohesion}}(\mu_1, \mu_2) = F_1(\mu_1) + F_2(\mu_2),
\]
where $F_1(z)$ and $F_2(z)$ are the Fr\'echet functions of $\S{C}_1$ and $\S{C}_2$, respectively. Thus, cluster cohesion sums the Fr\'echet variations within each cluster. Cluster separation can be defined by
\[
F_{\text{separation}}(\mu_1, \mu_2) = \dtw(\mu_1, \mu_2)^2.
\]
Cluster separation measures how well-separated or distinct two clusters are. A good clustering simultaneously minimizes cluster cohesion and maximizes cluster separation. Cluster cohesion is invariant under the choice of cluster means, because the Fr\'echet variations $F_1(\mu_1)$ and $F_2(\mu_2)$ are well-defined. The situation is different for cluster separation. From Eq.~\eqref{eq:warping-identification} follows $\dtw(\mu_1, \mu_2^1)^2 \leq \cdots \leq \dtw(\mu_1, \mu_2^r)$ for all $r \in \N$. Thus, cluster separation depends on the choice of mean $\mu_2$ of the second cluster $\S{C}_2$ (the mean $\mu_1$ of $\S{C}_1$ is unique). Figure \ref{fig:clus_sep} shows that cluster separation linearly increases with increasing number $r$ of replications. 
\begin{figure}[t]
\centering
\includegraphics[width=0.6\textwidth]{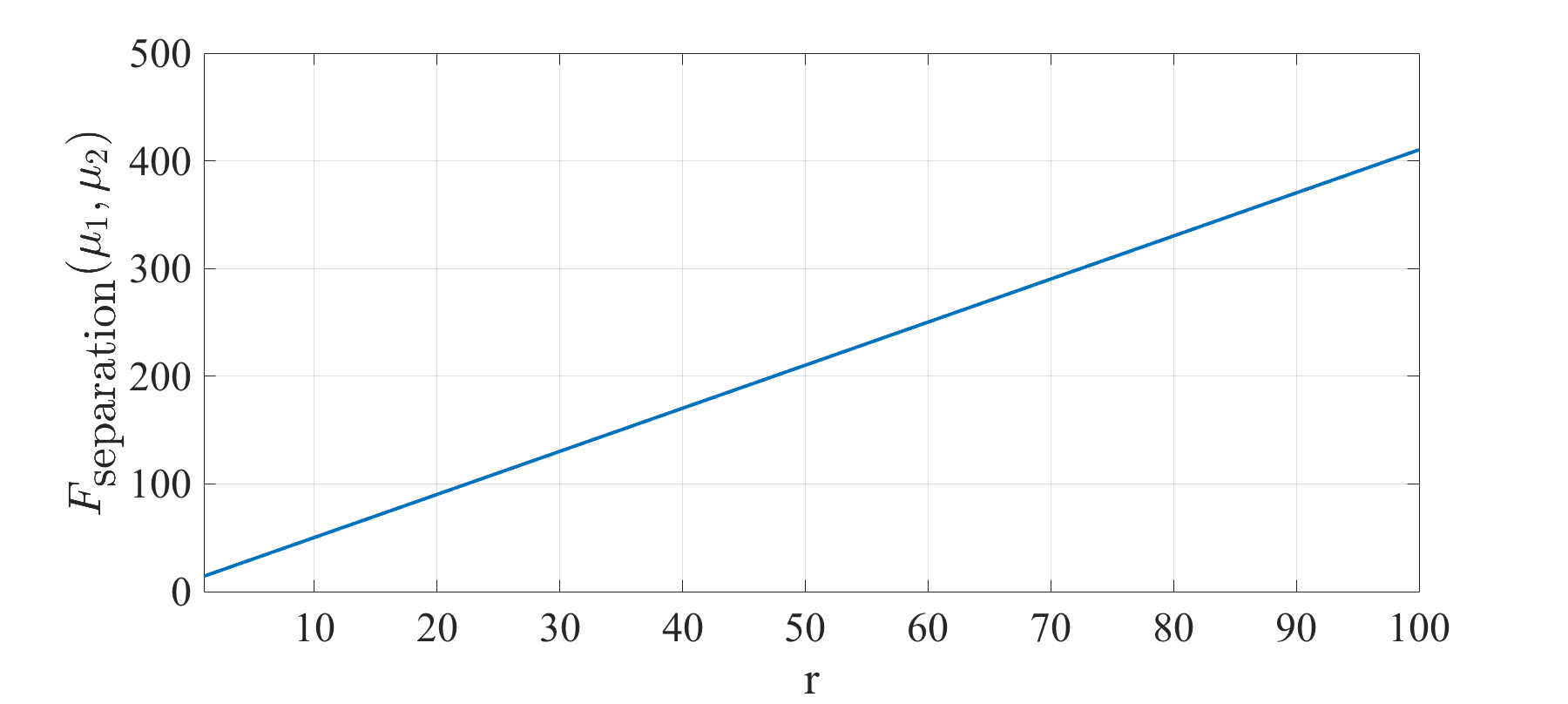}
\caption{Cluster separation $F_{\text{separation}}(\mu_1, \mu_2)$ as a function of the number $r$ of replications of the third element of $\mu^*_1 = (0.5, 2, 3)$.}
\label{fig:clus_sep}
\end{figure}
\QED
\end{example}

\medskip

Example \ref{example:non-condensed-k-means} illustrates that cluster cohesion is independent of the choice of mean but cluster separation can be maximized to infinity by expansions. In the given setting, cluster quality using cohesion and separation is not an inherent property of a clustering but rather a property of the chosen means as centroids, whose lengths can be controlled externally. Finally, we note that the scenario described in Example \ref{example:non-condensed-k-means} is not common but rather is exceptional \cite{Brill2019}. Nevertheless, subject to certain conditions, we can arbitrarily increase cluster quality without changing the clusters.

\section{Experiments}

In this experiment, we test the hypothesis of whether the twi-distance is more efficient than the dtw-distance without sacrificing for solution quality for a broad range of problems in nearest-neighbor classification. For this, we assess (i) the prevalence of reducible (non-irreducible) time series, (ii) the computation times, and (ii) the error-rates in nearest-neighbor classification

\subsection{UCR Benchmark Data}

We used the datasets from the 2015 -- 2018 version of the UCR Time Series Classification Repository \cite{Chen2015}.  

\subsection{Storage Consumption}

\begin{figure}[t]
\begin{subfigure}{0.4\textwidth} 
\includegraphics[width=\textwidth]{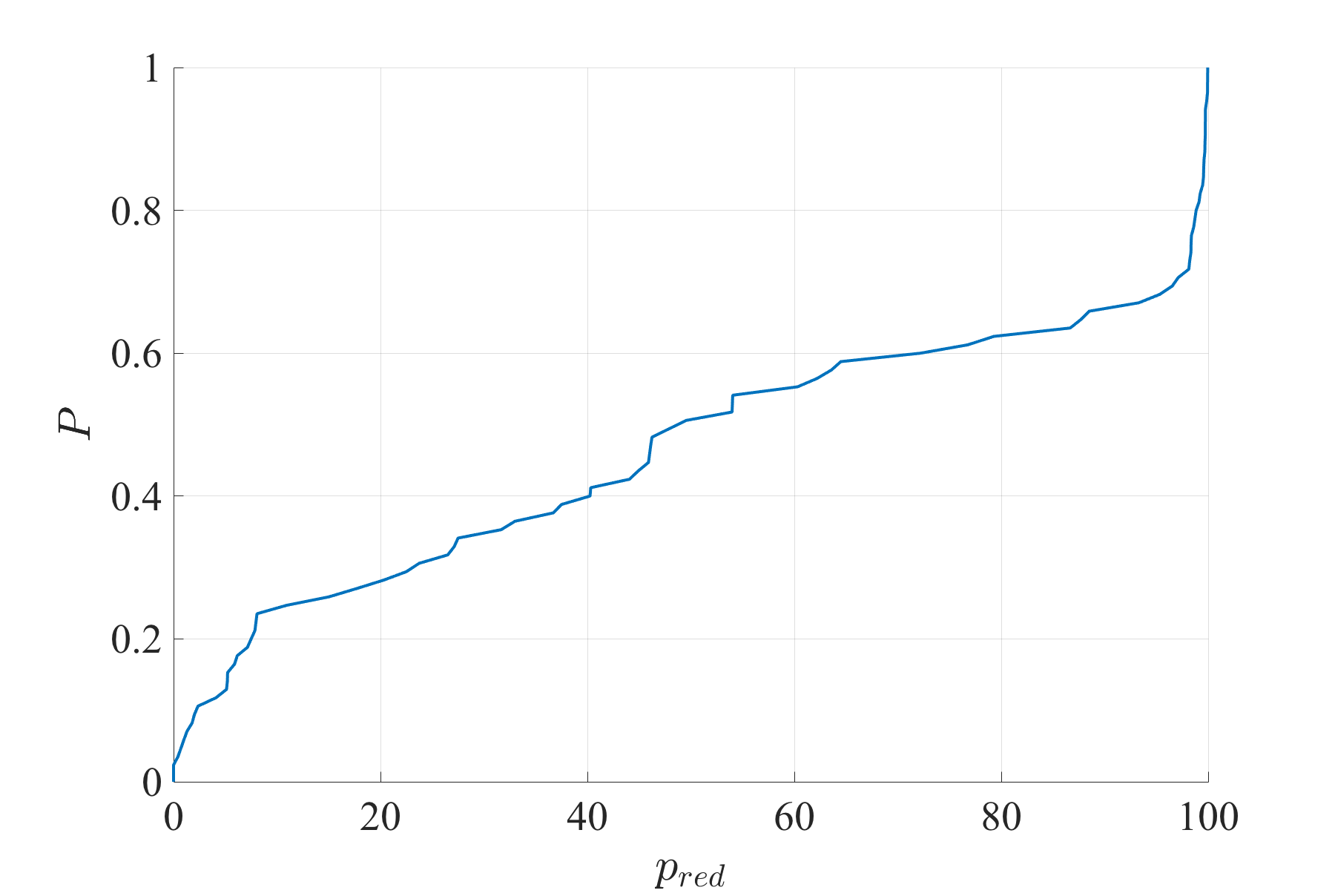}
\caption{Estimated cdf of $p_{\text{red}}$} 
\end{subfigure}
\hfill
\begin{subfigure}{0.4\textwidth}
\includegraphics[width=\textwidth]{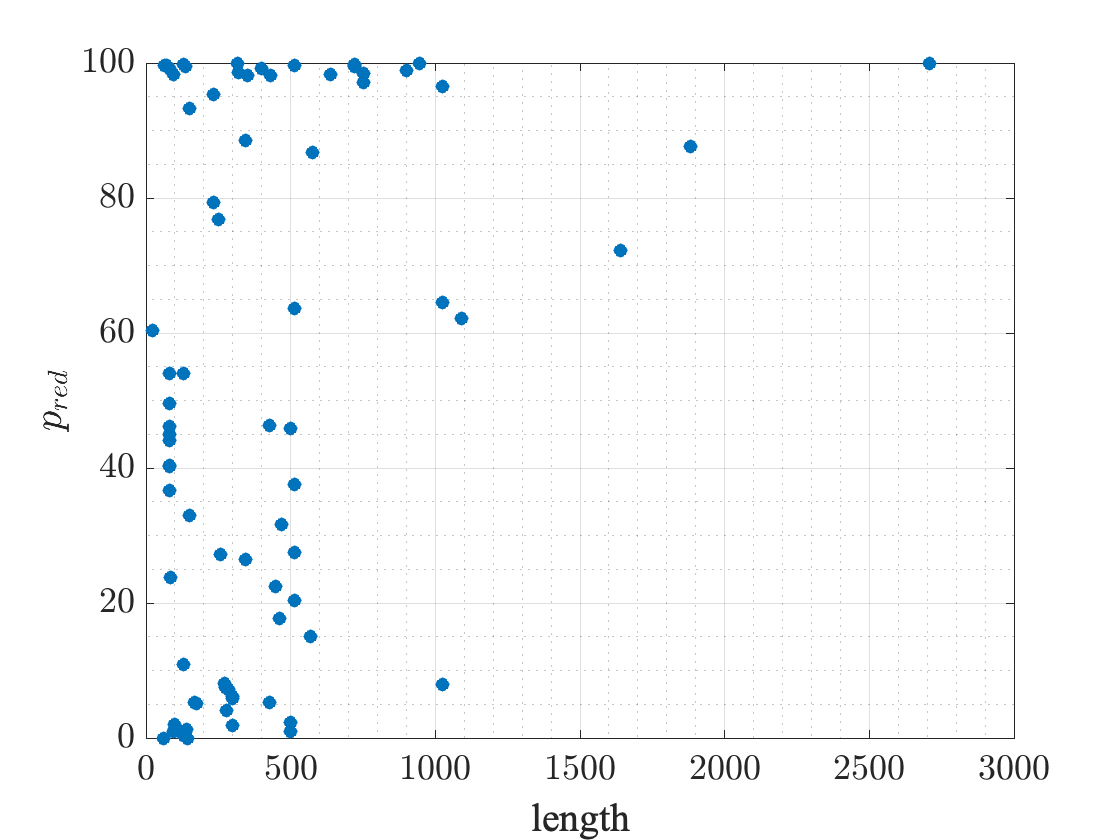}
\caption{length vs.~$p_{\text{red}}$} 
\end{subfigure}
\caption{Plot (a): Estimated cdf of the percentage $p_{\text{red}}$ of reducible time series. Plot (b): Length of time series vs.~$p_{\text{red}}$.}
\label{fig:pdf}
\end{figure}

In this section, we first study the prevalence of reducible time series. The prevalence of reducible time series provides us information on whether the twi-distance is suitable only for exceptional cases or of broader interest.  To answer this question, we computed the condensed form of every time series. A time series is reducible if it is longer than its condensed form. For every dataset, we recorded the percentage of reducible time series and the average number of deleted replications over the subset of reducible time series. 

\medskip

Figure \ref{fig:pdf}(a) shows the estimated cumulative distribution function of the percentage $p_{\text{red}}$ of reducible time series. The total number of time series over all $85$ datasets is $N = 135,\!793$. The percentage of reducible time series over all $N$ time series is $67.8 \%$. All but two out of $85$ datasets contain reducible time series. There are five datasets with at most $1 \%$ reducible time series, and $17$ datasets with at least $99 \%$ reducible time series. Among the reducible time series, the corresponding condensed forms are approximately $28.0 \%$ shorter on average. These results indicate that reducible time series are not an exceptional phenomenon but rather occur frequently. Consequently, application of the twi-distance is not restricted to exceptional cases but to a broad range of problems. 

Next, we study the relationship between the length $\ell$ of time series and the percentage $p_{\text{red}}$ of reducible time series. Figure \ref{fig:pdf}(b) shows a scatter plot of both quantities. We are interested to quantify the linear and the monotonic relationship between both variables. The Pearson correlation coefficient is $\rho= 0.32$ indicating a weak positive linear correlation.  Figure \ref{fig:pdf}(b) however suggests that this correlation could be caused by outliers. The  Spearman and Kendall correlation coefficients are $r = 0.28$ and $\tau = 0.18$, resp., indicating a weak to negligible monotonic positive monotonic correlation. We intuitively might expect that it is more likely to see two consecutive identical values in longer rather than in shorter time series. The results however do not convincingly support this intuition.

\subsection{Computation Time}

In this experiment, we compare the computation time of the twi-distance against other distances. 

\paragraph*{Experimental Protocol.}
We used the following distance functions: Euclidean (euc), dtw, bdtw, and twi.\footnote{We discuss the block dtw (bdtw) distance in Section \ref{subsec:compIssues}.}  We excluded  the \emph{StarLightCurves} and \emph{UWaveGestureLibraryAll} datasets for computational reasons. From every of the remaining $83$ UCR datasets, we randomly selected $100$ pairs of  time series. Then we computed the distance of each pair $100$-times and recorded the average time giving a total of $83 \times 100$ average times for every distance function.

\paragraph*{Results and Discussion.}

\begin{figure}[t]
	\begin{subfigure}{0.4\textwidth} 
		\includegraphics[width=\textwidth]{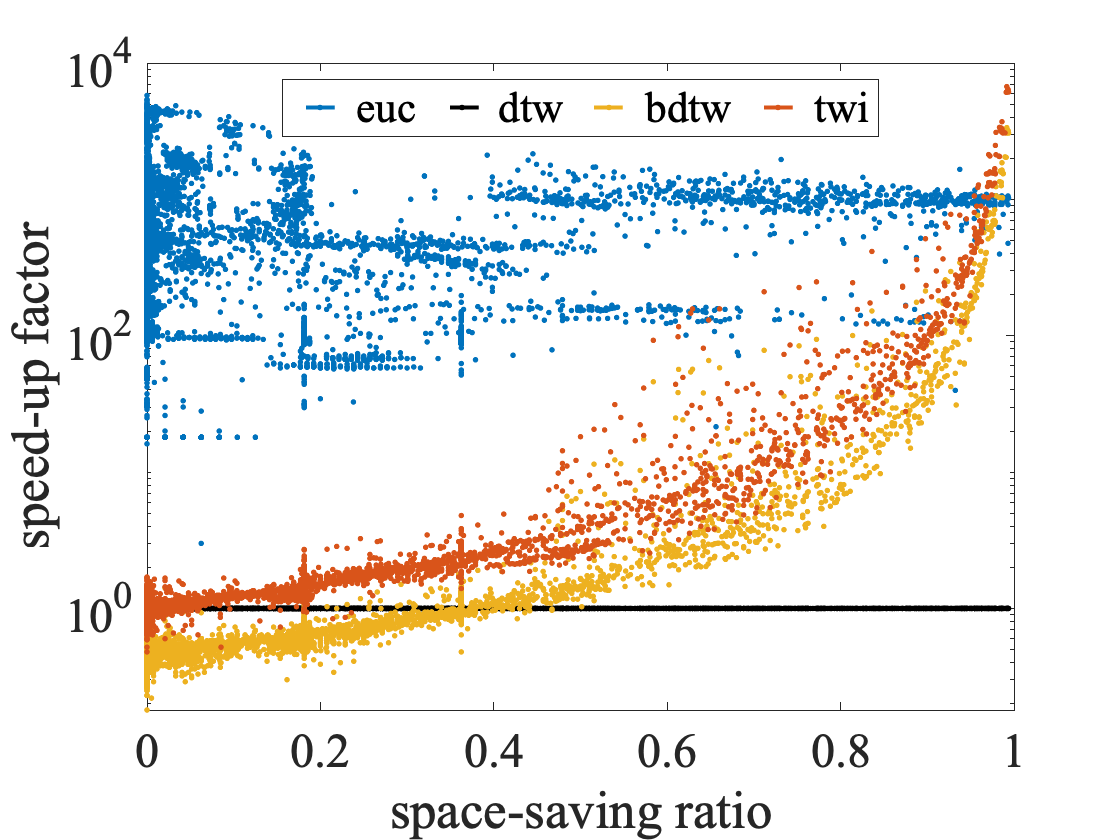}
		\caption{all vs.~dtw} 
	\end{subfigure}
	\hfill
	\begin{subfigure}{0.4\textwidth} 
		\includegraphics[width=\textwidth]{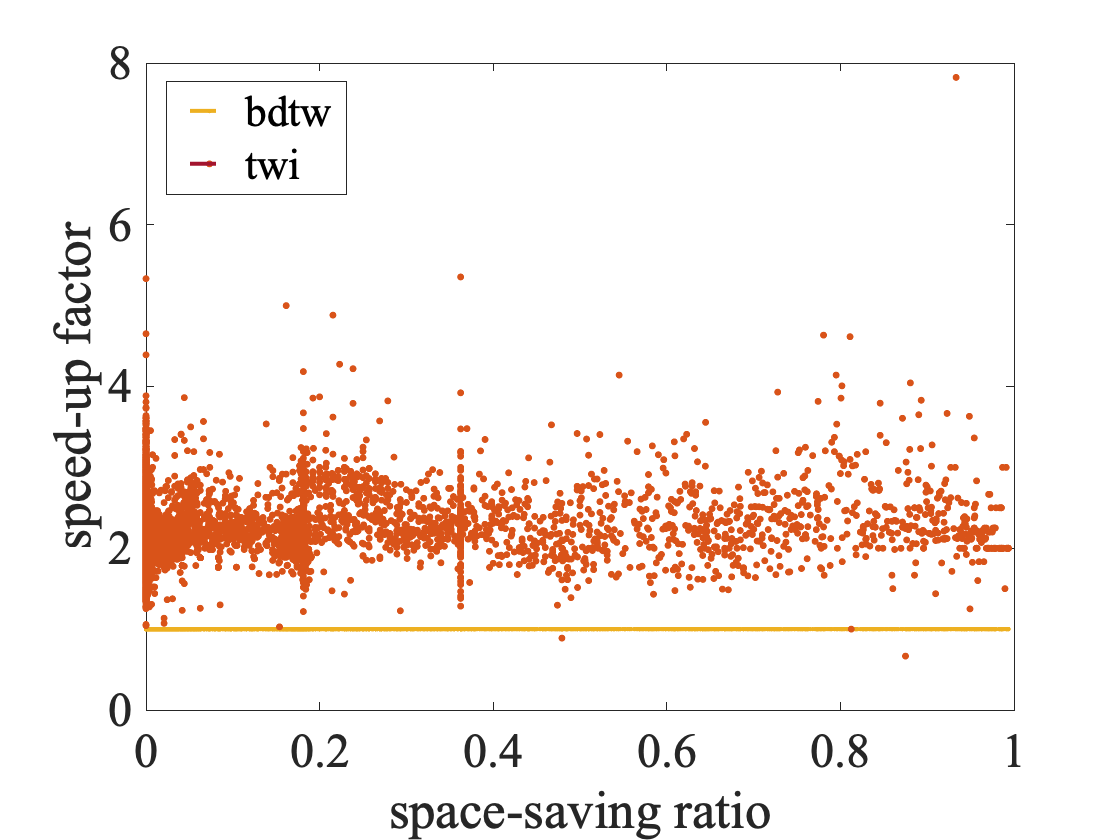}
		\caption{twi vs.~bdtw} 
	\end{subfigure}
\caption{Speed-up factor as a function of the space-saving ratio. The speed-up factor of Fig.~\ref{fig:ucr_times}(a) is in logarithmic scale.}
\label{fig:ucr_times}
\end{figure}

Figure \ref{fig:ucr_times} plots the computation times of every distance measure as a function of the (average) space-saving ratio. The computation times are shown as speed-up factors of all distances over the dtw-distance (left) and of the twi-distance over the bdtw-distance (right). The space-saving ratio for two time series $x$ and $y$ with condensed forms $x^*$ and $y^*$, resp., is defined by
\[
\rho_{\text{ss}} = 1-\frac{\abs{x^*}+\abs{y^*}}{2(\abs{x} + \abs{y})}.
\]
The shorter the condensed forms the more space can be saved with respect to the original time series. 

First, note that computing the twi-distance is never slower than computing the dtw- or bdtw-distance by construction. Variations in time, however, occur due to Java's just-in-time compilation. The speed-up factors of twi and bdtw over the dtw-distance grow with increasing space-saving ratio. At very high space-saving ratios computing the twi-distance is up to $6,000$ times faster than computing the dtw-distance and is even faster than computing the Euclidean distance. In addition, computing the twi-distance is about $2$ to $4$ times faster than computing the bdtw-distance. Computing the  bdtw-distance is slower, because additional operations are required to approximate the dtw-distance. The results suggest that the twi-distance is a more efficient alternative to the dtw-distance.

\subsection{Nearest-Neighbor Classification}

\renewcommand{\figurename}{Box}
\begin{figure}[b]
\centering
\colorbox{gr}{\textcolor{black}{
\begin{minipage}{\textwidth}
\footnotesize
\textbf{A Comment on the Standard Experimental Protocol of UCR Time Series Classification}
\\[0.5em]
\phantom{\ \ \ } 
It is common practice to apply hold-out validation using a pre-defined train-test split followed by a null hypothesis significance test to ensure statistical validity of the results.  There are strong arguments against both, hold-out validation on UCR datasets \cite{Gay2018} and null hypothesis significance testing \cite{Benavoli2018,Berrar2018}. 
\\[0.5em]
\phantom{\ \ \ } Currently, there is no established standard to ensure statistical validity of the results. As an alternative to null hypothesis significance testing, Bayesian hypothesis testing has been suggested \cite{Benavoli2018}. Instead of applying hold-out validation, resampling strategies such as bootstrap and  cross-validation are more likely to give \enquote{better} estimates of the expected error rate. However, according to \cite{Dau2018}, there is apparently at least one dataset for which resampling strategies can result in \enquote{twinning} training and test examples. We did not consider this dataset in our experiments. 
\\[-0.6em]
\end{minipage}}}
\vspace{-0.3cm}
\caption{}
\label{fig:experimental-protocol}
\end{figure}
\renewcommand{\figurename}{Figure}

In this experiment, we assess the performance of the twi-distance in nearest-neighbor (nn) classification.

\paragraph*{Experimental Protocol.}
We excluded  the \emph{StarLightCurves} and \emph{UWaveGestureLibraryAll} datasets for computational reasons. We applied $10$-fold cross-validation on the remaining $83$ datasets and recorded the average accuracy over all folds. Bayesian hypothesis testing was used to ensure statistical validity of the results. We compared the following distance functions to nn classification: Euclidean (euc), dtw, bdtw$_{\text{lb}}$, bdtw$_{\text{ub}}$, and twi, where bdtw$_{\text{lb}}$ and bdtw$_{\text{ub}}$ refer to the lower and upper bound of bdtw, respectively. Note that this experimental protocol does not follow the standard protocol in the literature. In Box \ref{fig:experimental-protocol}, we point to arguments why we applied a different protocol.

\paragraph*{Results and Discussion.}

\begin{table}[t]
\small
\begin{tabular}{l@{\qquad}rrrrr@{\qquad\qquad}l@{\qquad}rrrrr}
\toprule
&\multicolumn{5}{c}{row $\gg$ column}&&\multicolumn{5}{c}{row $\ll$ column}\\
&euc&dtw&bdtw$_{\text{lb}}$&bdtw$_{\text{ub}}$&twi&&euc&dtw&bdtw$_{\text{lb}}$&bdtw$_{\text{ub}}$&twi\\
\midrule
euc&0.0&36.1&32.5&34.9&32.5&euc&0.0&48.2&48.2&47.0&48.2\\
dtw&48.2&0.0&7.2&8.4&8.4&dtw&36.1&0.0&4.8&7.2&10.8\\
bdtw$_{\text{lb}}$&48.2&4.8&0.0&7.2&1.2&bdtw$_{\text{lb}}$&32.5&7.2&0.0&6.0&3.6\\
bdtw$_{\text{ub}}$&47.0&7.2&6.0&0.0&6.0&bdtw$_{\text{ub}}$&34.9&8.4&7.2&0.0&7.2\\
twi&48.2&10.8&3.6&7.2&0.0&twi&32.5&8.4&1.2&6.0&0.0\\
\bottomrule
\end{tabular}
\caption{Percentage of datasets for which row-classifier was practically better (left) / worse (right) than column-classifier.}
\label{tab:results-ucr-wins-losses}
\end{table}

Figure \ref{fig:ucr-results} shows the pairwise posteriors for the Bayesian sign-rank test with a rope of $0.5 \%$. Box \ref{fig:bayesian-test} briefly reviews Bayesian hypothesis testing and explains how to interpret the results in Figure \ref{fig:ucr-results}. Table \ref{tab:results-ucr-wins-losses} summarizes the percentages of pairwise practical wins and practical losses over all $83$ datasets. The meanings of practical wins and losses follow from the explanations in Box \ref{fig:bayesian-test}.

There are two main observations from the results of the Bayesian test. First, warping-based nn-classifiers are practically better than the euc-nn classifier with more than $98 \%$ probability.  Second, the warping-based nn-classifiers are practically equivalent. Table \ref{tab:results-ucr-wins-losses} shows that the euc-nn classifier is still a good and fast alternative for many problems and  that the twi-nn classifier has slight but not statistically validated advantages over all warping-based nn-classifiers.

The results are interesting because they challenge common practice in applying the dtw-distance: On the one hand, we postulate that the dtw-distance can cope with temporal variations. On the other hand, we store and process reducible time series with constant segments as if the dtw-distance is not capable to cope with temporal variation. This  inconsistency in applying the dtw-distance sacrifices space-saving and computational benefits without gaining substantially improved error rates in general. The dtw-nn classifier is better than the twi-nn classifier by a margin of at least $1\%$ only on four datasets with a maximum of $2.5\%$ on \emph{Lighting2}. Thus, there are a few datasets for which the dtw-distance should be preferred over the twi-distance. Consequently, the practical inconsistency of the dtw-distance lacks a general rationale in the particular case of nearest-neighbor classification.

\renewcommand{\figurename}{Box}
\begin{figure}[h]
\centering
\colorbox{gr}{\textcolor{black}{
\begin{minipage}{\textwidth}
\footnotesize
\textbf{Bayesian Hypothesis Testing}
\\[0.5em]
\phantom{\ \ \ } The description of Bayesian hypothesis testing follows \cite{Benavoli2018}.  Consider an experiment that consists of applying two classifiers $c_1$ and $c_2$ on $n$ problems. Suppose that $\mu$ is the mean difference of accuracies between both classifiers over all $n$ datasets. In Bayesian analysis, such an experiment is summarized by the posterior distribution  describing the distribution of $\mu$. By querying the posterior distribution, we can evaluate the probability that  $c_1$ is better/worse than $c_2$. Formally, $P(c_1 > c_2)$ is the posterior probability that $\mu > 0$. In other words, $P(c_1 > c_2)$ is  the probability that $c_1$ is better than $c_2$. Note that such a statement is not possible with null-hypothesis tests. The null-hypothesis test only makes a statement on the probability of the data given that the classifiers are the same and does not say anything about the probability that $c_1$ is better than $c_2$.
\\[0.2em]
\phantom{\ \ \ } Depending on the application problem, we may regard tiny differences in accuracy between $c_1$ and $c_2$ as negligible in practice.  Though negligible in practice,  the smallest differences can be made statistically significant in null-hypothesis testing by collecting enough data. Thus, statistical significance does not imply practical significance. With Bayesian tests, we can decide for a region of \emph{practical equivalence} defined by a problem-dependent \emph{rope} $r \geq 0$. Then the posterior $P(c_1 = c_2)$ is the probability that $\abs{\mu}  \leq r$, that is the probability that both classifiers are practically equivalent. In addition, $P(c_1 \gg c_2)$ is the probability that $c_1$ is practically better than $c_2$ and $P(c_1 \ll c_2)$ is the probability that $c_1$ is practically worse than $c_2$.
\\[0.2em]
\phantom{\ \ \ } Figure \ref{fig:ucr-results} visualizes the pairwise posteriors for the Bayesian sign-rank test with a rope of $0.5 \%$ for every pair of distances. The vertices of a triangle represent the rope (top), classifier $c_1$ (bottom-left), and classifier $c_2$ (bottom-right). Each vertex is associated with a region.  A region represents the case where its associated vertex is more probable than both other vertices together. The fraction of points falling into a region estimates the probability of the associated vertex. Roughly, the points are generated by Monte Carlo sampling from a Dirichlet distribution to obtain the posterior distribution.
\\[-0.7em]
\end{minipage}}}
\vspace{-0.3cm}
\caption{}
\label{fig:bayesian-test}
\end{figure}
\renewcommand{\figurename}{Figure}

\begin{figure}[t]
	\begin{subfigure}{0.24\textwidth} 
		\includegraphics[width=\textwidth,trim= 3.5cm 1.5cm 3cm 2.5cm,clip]{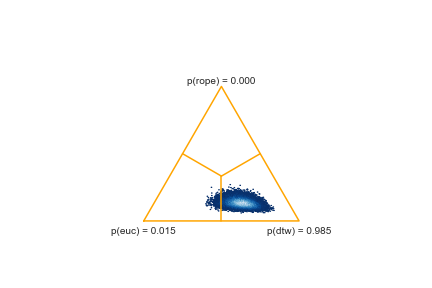}
		\caption{euc vs.~dtw} 
	\end{subfigure}
	\begin{subfigure}{0.24\textwidth} 
		\includegraphics[width=\textwidth,trim= 3.5cm 1.5cm 3cm 2.5cm,clip]{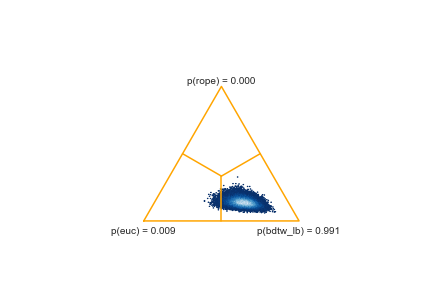}
		\caption{euc vs.~bdtw$_{\text{lb}}$} 
	\end{subfigure}
	\begin{subfigure}{0.24\textwidth} 
		\includegraphics[width=\textwidth,trim= 3.5cm 1.5cm 3cm 2.5cm,clip]{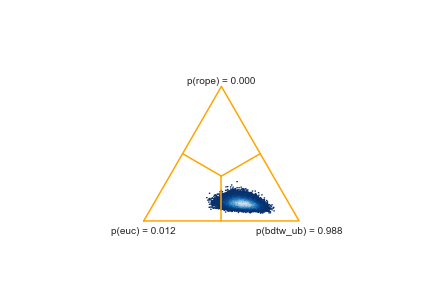}
		\caption{euc vs.~bdtw$_{\text{ub}}$} 
	\end{subfigure}
	\begin{subfigure}{0.24\textwidth} 
		\includegraphics[width=\textwidth,trim= 3.5cm 1.5cm 3cm 2.5cm,clip]{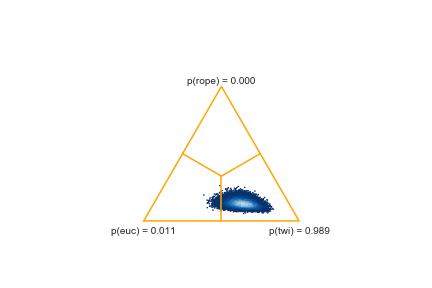}
		\caption{euc vs.~twi} 
	\end{subfigure}
	\begin{subfigure}{0.24\textwidth} 
		\includegraphics[width=\textwidth,trim= 3.5cm 1.5cm 3cm 2.5cm,clip]{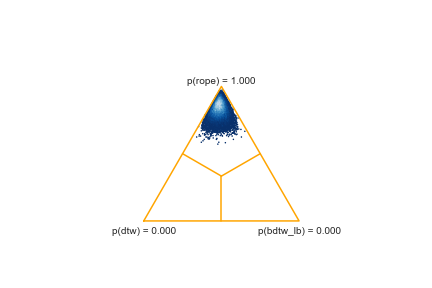}
		\caption{dtw vs.~bdtw$_{\text{lb}}$} 
	\end{subfigure}
	\begin{subfigure}{0.24\textwidth} 
		\includegraphics[width=\textwidth,trim= 3.5cm 1.5cm 3cm 2.5cm,clip]{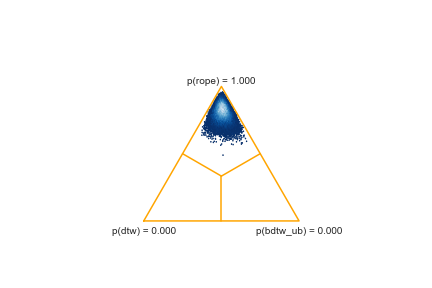}
		\caption{dtw vs.~bdtw$_{\text{ub}}$} 
	\end{subfigure}
	\begin{subfigure}{0.24\textwidth} 
		\includegraphics[width=\textwidth,trim= 3.5cm 1.5cm 3cm 2.5cm,clip]{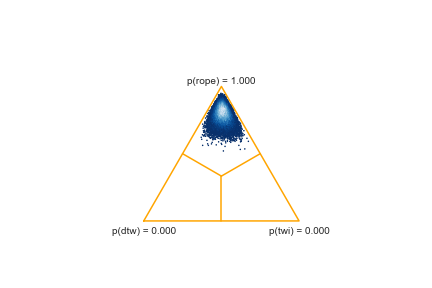}
		\caption{dtw vs.~twi} 
	\end{subfigure}
	\begin{subfigure}{0.24\textwidth} 
		\includegraphics[width=\textwidth,trim= 3.5cm 1.5cm 3cm 2.5cm,clip]{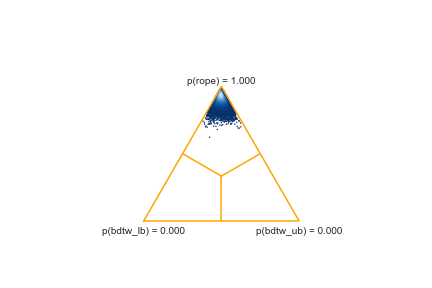}
		\caption{bdtw$_{\text{lb}}$ vs.~bdtw$_{\text{ub}}$} 
	\end{subfigure}
	\begin{subfigure}{0.24\textwidth} 
		\includegraphics[width=\textwidth,trim= 3.5cm 1.5cm 3cm 2.5cm,clip]{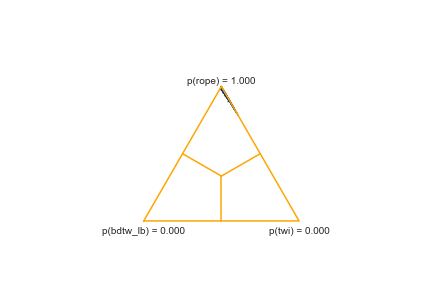}
		\caption{bdtw$_{\text{lb}}$ vs.~twi} 
	\end{subfigure}
	\begin{subfigure}{0.24\textwidth} 
		\includegraphics[width=\textwidth,trim= 3.5cm 1.5cm 3cm 2.5cm,clip]{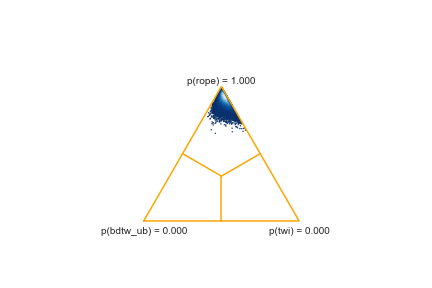}
		\caption{bdtw$_{\text{ub}}$ vs.~twi} 
	\end{subfigure}
	\begin{subfigure}{0.48\textwidth} 
	\end{subfigure}
	\caption{Posterior of first classifier vs.~second classifier for the Bayesian sign-rank test using a rope of $0.5 \%$.} 
	\label{fig:ucr-results}
\end{figure}

\subsection{Power Consumption Classification}

In this experiment, we assess the performance of the twi-nn classifier on  power consumption data consisting of long time series with large constant segments.

\paragraph*{Data.}
The 2nd version of the  Almanac of Minutely Power dataset (AMPds2) provided by \cite{Makonin2016} includes electricity measurements of a house in the Greater Vancouver metropolitan area in British Columbia (Canada).\footnote{URL of AMPds2: \url{https://dataverse.harvard.edu/dataset.xhtml?persistentId=doi:10.7910/DVN/FIE0S4}.} The dataset contains electricity measurements over a period of two years at one minute intervals from $21$ power sub-meters giving a total of $1,051,200$ readings from each sub-meter. In this study, we consider the real power P measured in VAR. 

\begin{table}[ph]
\centering
\begin{tabular}{l@{\qquad}rrrr}
\toprule
&1d&7d&14d&30d\\
\midrule
length of time series & $1,440$ & $10,080$ & $20,160$ & $43,200$\\
train (test) size & 6935 & 997 & 503 & 237\\
space-saving rate & 73.8 & 73.8 & 73.8 & 73.8 \\
\bottomrule
\end{tabular}
\caption{Characteristic properties of the AMPds2 dataset for different numbers $d$ of days.}
\label{tab:ampds2-characteristics}
\end{table}

\paragraph*{Distance Functions.}
We applied the following distance functions to nearest-neighbor classification: euc, dtw, bdtw$_{\text{lb}}$, bdtw$_{\text{ub}}$, twi, opt-dtw, and opt-twi. The opt-distances optimize nearest-neighbor search by applying the following techniques: (i) Sakoe-Chiba band \cite{Sakoe1978} of radius $0.1$, (iii) lower-bounding with LB\_Keogh \cite{Keogh2002} followed by LB\_Lemire \cite{Lemire2009}, and (iii) early abandoning during dtw-computation. To apply the lower-bounding techniques, query and candidate time series should be of the same length. We aligned the candidates to the query either by truncation or expansion.

\paragraph*{Experimental Protocol.}
The task is to assign segments of power readings to their corresponding sub-meters. The sub-meters form $21$ classes. We split the 100K readings of each sub-meter $s$ into sequences lasting $d$ days, where $d \in \cbrace{1, 7, 14, 30}$ giving a total of $N_d(s)$  time series $x_1(s), \ldots, x_{N_d}(s)$.  We partitioned the $N_d$ time series into a training and test set. Time series with odd index went into the training set and all other time series went into the test set. Table \ref{tab:ampds2-characteristics} shows the length of the time series and the sizes of the training and tests. For every $d$ and every distance measure, we applied the nearest-neighbor classifier and recorded the error rate and the total time to classify all test examples.

\paragraph*{Results and Discussion.}

Table \ref{tab:results-ampds2} shows the error rates and computation times of the nearest-neighbor classifiers. Note that the computation times are wall-clock times taken on a multi-user machine and are therefore only rough upper bounds of the actual computation times.

\begin{table}[t]
\centering
\begin{tabular}{l@{\qquad}lrrrrrrrr}
\toprule
& \multicolumn{4}{c}{error rate} & & \multicolumn{4}{c}{time (hh:mm:ss)}\\
&1d&7d&14d&30d& \qquad &1d&7d&14d&30d\\
\midrule
euc&51.68&74.92&75.94&75.53&&0:24:29&0:03:30&0:01:47&0:00:48\\
dtw&20.36&6.52	&7.36&5.91	&&65:04:21&66:54:09&67:58:13&69:10:28\\
bdtw$_{\text{lb}}$&20.53&6.12&6.76&5.91&&11:58:07&11:48:49&11:51:05&11:49:37\\
bdtw$_{\text{ub}}$&20.53&6.42&6.96&5.91&&11:54:24&12:03:06&12:12:36&12:04:41\\
twi&20.69&6.72&7.55&5.91&&6:15:32&6:32:18&6:51:55&6:52:44\\
opt-dtw&20.36&6.52&7.36&5.91&&26:14:50&33:23:09&35:08:48&35:40:14\\
opt-twi&20.69&6.72&7.55&5.91&&2:24:52&2:57:00&3:12:45&3:14:02\\
\bottomrule
\end{tabular}
\caption{Error rates and computation time of seven nearest-neighbor classifiers on AMPds2 data.}
\label{tab:results-ampds2}
\end{table}

The first observation is that the Euclidean distance and the dtw-distance are unsuitable for the given classification problem. The Euclidean distance is fast to compute (a minute up to half an hour) but results in too high error rates ($50\%-75\%$), whereas the dtw-distance results in substantially lower error rates ($6\%-20\%$) but is slow to compute (almost three days). Even the opt-dtw-nn classifier requires up to one and a half day to complete its computations which is about three times slower than the next slowest classifier (bdtw-nn). 

The second observation is that the error rates of the warping-based classifiers are comparable with slight but not statistically validated advantages of more than $0.5\%$ for the bdtw$_{\text{lb}}$-nn classifier on the $14$ days dataset.

The third observation is that the twi-nn and opt-twi-nn classifiers are fastest among the warping-based classifiers. The twi-nn classifier is about twice faster than the  bdtw-nn classifiers, about four to five times faster than the opt-dtw-nn classifier and more than ten times faster than the dtw-classifier. The speed-up factor of the opt-twi-nn classifier over the twi-nn classifier is $3$. The space saving factor of twi over dtw is $73.8\%$ (see Table \ref{tab:ampds2-characteristics}).

These findings essentially confirm the findings on the UCR datasets and suggest that the twi-distance is a fast and space-saving alternative to the dtw-distance and its bdtw-approximations without substantially sacrificing solution quality. In addition, the results indicate that even optimized nearest-neighbor search remains an open issue for warping-based distances. In such cases, we could resort to heuristics such as FastDTW \cite{Salvador2007} that can be also applied to irreducible time series, that is as approximations to the twi-distance.

\subsection{Limitations of the DTW- and TWI-Distance}

In this section, we study basic limitations of the dtw- and twi-distance in nearest-neighbor classification.

\paragraph*{Data.}
We generated a synthetic dataset $\S{D}$ consisting of $200$ time series of length $100$ from two classes $\S{Y} = \cbrace{\pm 1}$ such that each class has $100$ elements. Figure \ref{fig:fail} shows examples of time series from both classes and their corresponding condensed forms. The elements $x_i$ of a time series $x = (x_1, \ldots, x_{100})\in \S{D}$ of class $y \in \S{Y}$ are of the form
\begin{align}\label{eq:synthetic-data}
x_i = \begin{cases}
b_y + \varepsilon_b & i \notin \S{I}(x)\\
c_y + \varepsilon_c & i \in \S{I}(x)
\end{cases},
\end{align}
where $b_y$ is the class-dependent base value of the constant segment, $c_y$ is the class-dependent height of the cylinder, $\S{I}(x) \subset [100]$ is a randomly selected interval of fixed length $10$, and $\varepsilon_{b}, \varepsilon_c \in U(\theta)$ are uniformly distributed noise terms with values from the interval $[0, \theta]$. 

\begin{figure}[t]
\includegraphics[width=0.45\textwidth]{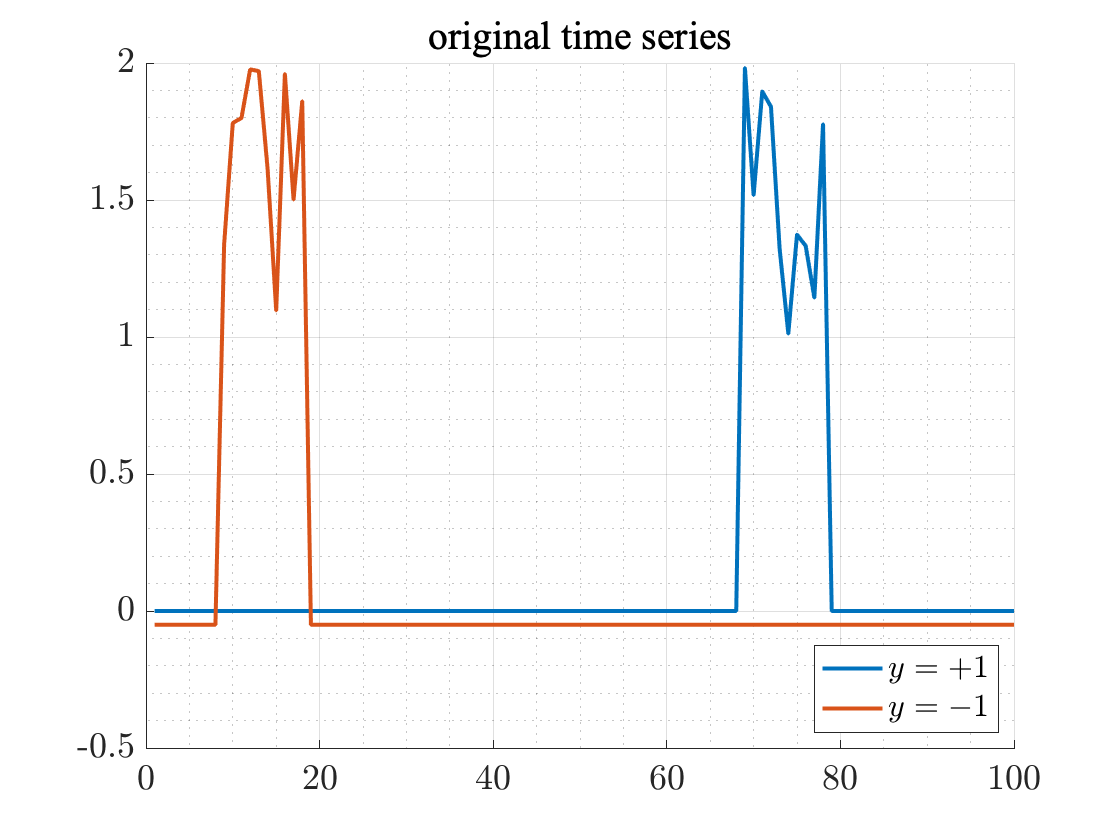}
\hfill
\includegraphics[width=0.45\textwidth]{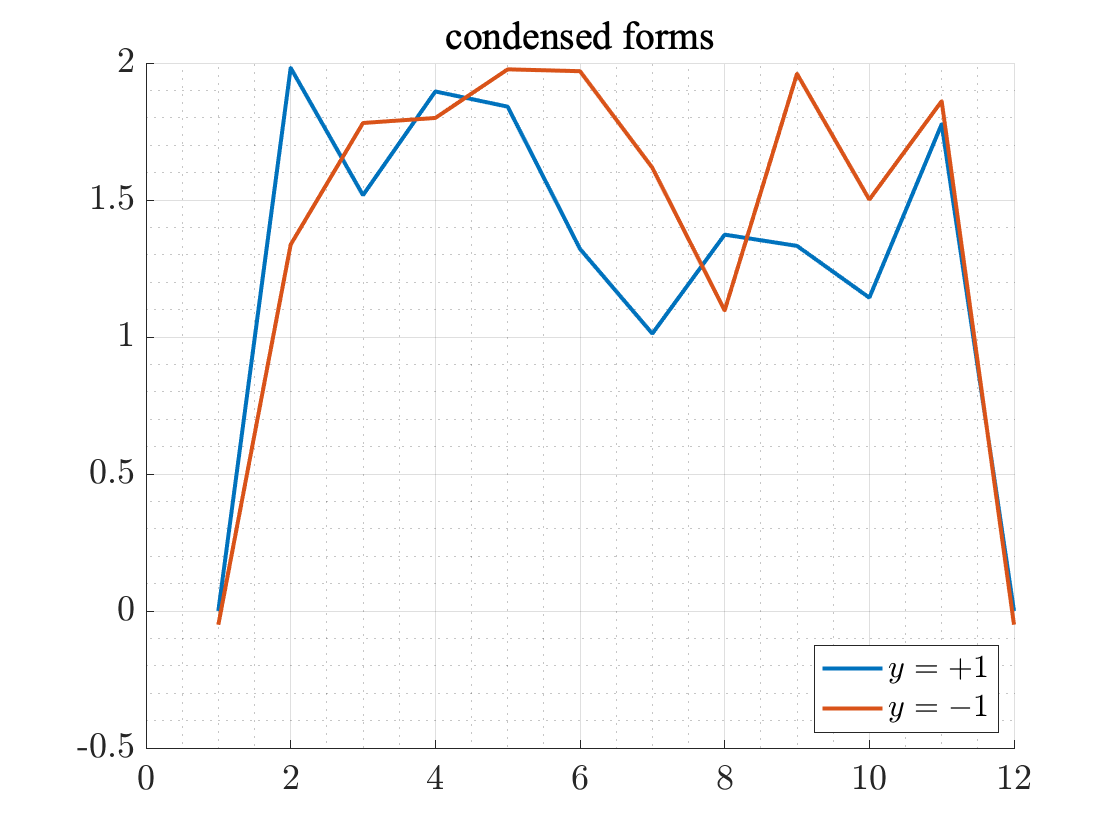}
\caption{Synthetic time series of both classes (left) and their corresponding condensed forms (right).}
\label{fig:fail}
\end{figure}

\paragraph*{Experimental Protocol.}
A single experiment generates a dataset according to Eq.~\eqref{eq:synthetic-data} and then splits the dataset into a train and test set with balanced class distribution. Finally, we applied the nearest-neighbor classifier using the euc-, dtw-, and twi-distance. We repeated this experiment $50$-times and reported the average error rates of each classifier.

\paragraph*{Results and Discussion.}
Table \ref{tab:res-limitations} presents the results for five different configurations of the data generating parameters. The overall observation is that the euc-nn classifier failed because it is sensitive to temporal variations of the cylinder. 

\medskip

\noindent
Experiment E1 \& E2:
The twi-nn classifier failed in experiment E1 with an error rate of $37 \%$, because the base value $b_y$ of the constant segment is the discriminative feature to distinguish between both classes but the noise in the cylinder contributed most to the twi-distance.  In contrast, the dtw-distance perfectly separated both classes because the accumulated differences between the constant segments from different classes contributed most to the distance and outweighed the noisy differences of the cylinders. These findings indicate that the twi-distance can fail when differences in constant segments matter. In light of the failure of the euc-nn classifier, this experiment provides an example for which warping is important but warping-invariance is not desired. 

The noise $\varepsilon_c$ in experiment E1 is high because it can take values up to $100\%$ of the cylinder height $c_y$. In experiment E2, we used the same configuration as in experiment E1 but reduced the maximum contribution of the  noise term $\varepsilon_c$  to $10\%$ of $c_y$. The results show that both warped-nn classifiers perfectly separate both classes. The findings of both experiments indicate that the twi-nn classifier is more sensitive to noise than the dtw-nn classifier given that the base value  $b_y$  substantially contribute as discriminant between both classes.

\medskip

\noindent
Experiment E3 \& E4: In experiment E1 \& E2, the base value $b_y$ was the discriminative feature, whereas the cylinder heights $c_y$ were identical. In experiments E3 \& E3, we reversed the roles of $b_y$ and $c_y$. In experiment E3, the dtw-nn classifier failed with an error rate of $ 31\%$ and the twi-nn classifier perfectly separated both classes due to the opposite reasons as in experiment E1. These results suggest that the dtw-distance can fail when differences in constant segments are irrelevant. This experiment provides an example for which warping alone is insufficient but warping-invariance matters. 

In experiment E4, we reduced the noise $\varepsilon_b$ to $\theta_b = 0.1$. For similar reasons as in experiment E2, we observe that both warped-nn classifiers perfectly separate both classes.  

\medskip

\noindent
Experiment E5:
In time series classification, it is common practice to normalize the data. In experiment E6, we applied a z-transformation on the time series from experiment E1. The euc-nn and twi-nn classifiers failed for the same reasons as in the previous experiment. The dtw-nn classifier failed, because the z-transformation eliminated the differences between constant segments such that both classes became indistinguishable. This experiment is one of the rare examples, where normalization deteriorates discriminative features

\bigskip

\begin{table}
\centering
\begin{tabular}{r@{\quad}cccrlcc@{\quad}rrr}
\toprule
id & $b_{+1}$ & $b_{-1}$ & $\theta_b$ & $c_{+1}$ & $c_{-1}$ & $\theta_c$ & z-transform& euc-nn & dtw-nn & twi-nn\\
\midrule
E1& 0 & $-$0.05 & 0.0 & 1.0 & 1.0 & 1.0 & no & 50 \% & 0 \%& 37 \%\\ 
E2& 0 & $-$0.05 & 0.0 & 1.0 & 1.0 & 0.1 & no & 35 \% & 0 \%& 0 \%\\
\midrule
E3& 0 & 0 & 1.0 & 1.0 & 1.05 & 1.0&  no & 50 \% & 31 \% & 0 \%  \\
E4& 0 & 0 & 0.1 & 1.0 & 1.05 & 1.0&  no & 49 \% & 0 \% & 0 \%  \\
\midrule
E5& 0 & $-$0.05 & 0.0 & 1.0 & 1.0 & 1.0 & yes & 51 \% & 45 \%& 45 \%\\ 
\bottomrule
\end{tabular}
\caption{Error rates of the three nn-classifiers on different variants of the cylinder-peak data. .}
\label{tab:res-limitations}
\end{table}
\bigskip

To summarize, there are cases for which the twi-distance is inappropriate and not competitive with the dtw-distance and vice versa. The twi-distance is inappropriate when constant segments are crucial for discriminating both classes but are outweighed by other factors. Conversely, the dtw-distance is inappropriate when constant segments are irrelevant and there differences are not outweighed by relevant factors. We could not detect either case in the UCR archive. 

\section{Conclusion}

The dtw-distance fails to satisfy the triangle inequality and the identity of indiscernibles. As a consequence, the dtw-distance is not warping-invariant though otherwise stated in some publications. The lack of warping-invariance of the dtw-distance contributes to strange behavior and practical inconsistencies. For example, the strange behavior of the 
nearest-neighbor rule affects other time series mining methods such as k-means clustering. Subject to certain conditions, we can arbitrarily increase the cluster quality of k-means without changing the clusters. The practical inconsistency refers to the discrepancy between the underlying idea of the dtw-distance and how the dtw-distance is actually used in applications. The literature postulates that the dtw-distance can cope with temporal variations but stores and processes time series in a form as if the dtw-distance cannot cope with such variations. In doing so, the dtw nearest-neighbor classifier sacrifices storage-savings and computational efficiency without gaining substantially better error rates. The proposed twi-distance is a warping-invariant semi-metric that resolves both issues, the strange behavior of the dtw-distance and its practical inconsistency. The twi-distance requires less storage and is faster to compute than the dtw-distance without sacrificing for solution quality in nearest-neighbor classification for a broad range of problems. A limitation of the twi-distance are reducible time series whose constant segments are important discriminative features for classification. We could not detect such features among the $83$ datasets from the UCR archive. The proposed contribution challenges past use of the dtw-distance in nearest-neighbor classification and instead suggests the more efficient and theoretically more consistent twi-distance as default measure. Future research aims at studying to what extent the findings of this article carry over to time series mining techniques other than nearest-neighbor classification.

\paragraph*{\textbf{Acknowledgements}.} B.~Jain was funded by the DFG Sachbeihilfe \texttt{JA 2109/4-2}.

\appendix
\small
\section{Semi-Metrification of DTW-Spaces}\label{sec:theory}

This section converts the dtw-space to a semi-metric space and shows that the canonical extension of the proposed semi-metric is warping-invariant. The technical treatment is structured as follows: Section \ref{subsec:words} studies expansions, compressions, condensed forms, and irreducibility in terms of words over arbitrary alphabets. Sections \ref{subsec:warping-walks} and \ref{subsec:properties-of-warping-walks} propose warping walks as a more general and convenient concept than warping paths. In addition, we study some properties of warping walks. Section \ref{subsec:semi-metrification} constructs a semi-metric quotient space induced by the dtw-space as described in Section \ref{sec:results}. Finally, Section \ref{subsec:warping-invariance} proves that the canonical extension of the quotient metric is warping-invariant. 

For the sake of readability, this section is self-contained and restates all definitions and notations mentioned in earlier sections. Throughout this section, we use the following notations: 

\begin{notation}
The set $\R_{\geq 0}$ is the set of non-negative reals, $\N$ is the set of positive integers, and $\N_0$ is the set of non-negative integers. We write $[n] = \cbrace{1, \ldots, n}$ for $n \in \N$.
\end{notation}

\subsection{Prime Factorizations, Expansions, and Condensed Forms of Words}\label{subsec:words}

The key results of this article apply the same auxiliary results on warping paths and on time series. Both, warping paths and time series, can be regarded as words over different alphabets. Therefore, we derive the auxiliary results on the more general structure of words. We first propose the notions of prime words and prime factorization of words. Then we show that every word has a unique prime factorization. Using prime factorizations we define expansions, compressions, irreducible words, and condensed forms. Finally, we study the relationships between these different concepts. 

\medskip

Let $\S{A}$ be a set, called \emph{alphabet}. We impose no restrictions on the set $\S{A}$. Thus, the elements of $\S{A}$ can be symbols, strings, trees, graphs, functions, reals, feature vectors, and so on. A \emph{word} over $\S{A}$ is a finite sequence $x = x_1 \cdots x_n$ with elements $x_i \in \S{A}$ for all $i \in [n]$. The set $\S{A}(x) = \cbrace{x_1, \ldots, x_n}$ is the set of elements contained in $x$. We denote by $\S{A}^*$ the set of all words over $\S{A}$ and by $\varepsilon$ the empty word. We write $\S{A}^+ = \S{A} \setminus \cbrace{\varepsilon}$ to denote the set of non-empty words over $\S{A}$.

Let $x = x_1 \cdots x_m$ and $y = y_1 \cdots y_n$ be two words over $\S{A}$. The \emph{concatenation} $z = xy$ is a word $z = z_1 \cdots z_{m+n}$ with elements
\[
z_i = \begin{cases}
x_i & 1 \leq i \leq m\\
y_{m-i} & m < i \leq m+n
\end{cases}
\]
for all $i \in [m+n]$. The concatenation is an associative operation on $\S{A}^*$. Hence, we can omit parentheses and write $xyz$ instead of $(xy)z$ or $x(yz)$ for all $x, y, z \in \S{A}^*$. For a given $n \in \N_0$, we write $x^n$ to denote the $n$-fold concatenation of a word $x \in \S{A}^*$. We set $x^0 = \varepsilon$ for every $x \in \S{A}^*$ and $\varepsilon^n = \varepsilon$ for all $n \in \N_0$. 

The \emph{length} of a word $x$ over $\S{A}$, denoted by $\abs{x}$, is the number of its elements. A \emph{prime word} is a word of the form $x = a^n$, where $a \in \S{A}$ and $n \in \N$ is a positive integer. From the definition follows that $\S{A}(a^n) = \cbrace{a}$ and that the empty word $\varepsilon$ is not a prime word. 

A non-empty word $v \in \S{A}^+$ is a \emph{factor} of a word $x \in \S{A}^*$ if there are words $u, w \in \S{A}^*$ such that $x = uvw$. Every non-empty word $x$ is a factor of itself, because $x = \varepsilon x \varepsilon$. In addition, the empty word $\varepsilon$ has no factors. 

Let $\S{A}^{**} = (\S{A}^*)^*$ be the set of all finite words over the alphabet $\S{A}^*$. A \emph{prime factorization} of a word $x \in \S{A}^*$ is a word $\Pi(x) = p_1 \cdots p_k \in \S{A}^{**}$ of prime factors $p_t \in \S{A}^*$ of $x$ for all $t \in [k]$ such that 
\begin{enumerate}
\itemsep0em
\item $p_1 \cdots p_k = x$ \hfill (\emph{partition-condition})
\item $\S{A}(p_t) \neq \S{A}(p_{t+1})$ for all $t \in [k-1]$ \hfill (\emph{maximality-condition})
\end{enumerate}
The partition-condition demands that the concatenation of all prime factors yields $x$. The maximality-condition demands that a prime factorization of $x$ consists of maximal prime factors. By definition, we have $\Pi(\varepsilon) = \varepsilon$.

\begin{proposition}
Every word over $\S{A}$ has a unique prime factorization.
\end{proposition}

\begin{proof}
Since $\Pi(\varepsilon) = \varepsilon$, it is sufficient to consider non-empty words over $\S{A}$. We first show that every $x \in \S{A}^+$ has a prime factorization by induction over the length $\abs{x} = n$. 

\medskip

\noindent
\emph{Base case}: Suppose that $n = 1$. Let $x$ be a word over $\S{A}$ of length one. Then $x$ is prime and a factor of itself. Hence, $\Pi(x) = x$ is a prime factorization of $x$.

\medskip

\noindent
\emph{Inductive step}: Suppose that every word over $\S{A}$ of length $n$ has a prime factorization. Let $x = x_1 \cdots x_{n+1}$ be a word of length $\abs{x} = n+1$. By induction hypothesis, the word $x' = x_1 \cdots x_n$ has a prime factorization $\Pi(x') = p_1 \cdots p_k$ for some $k \in [n]$. Suppose that $p_k = a^{n_k}$ for some $n_k \in \N$. Then $\S{A}(p_k) = a$. We distinguish between two cases: 
\begin{enumerate}
\itemsep0em
\item 
Case $a = x_{n+1}$: The concatenation $\tilde{p}_k = p_k a = a^{n_k + 1}$ is a prime word with $\S{A}(\tilde{p}_k) = \cbrace{a}$. Let $q = p_1\cdots p_{k-1}$. We have 
\[
x = x'a = p_1 \cdots p_k a \stackrel{(1)}{=} p_1 \cdots p_{k-1}\tilde{p}_k \stackrel{(2)}{=} q \tilde{p}_k \varepsilon.
\] 
Equation (1) shows the first property of a prime factorization and equation (2) shows that $\tilde{p}_k$ is a factor of $x$. From the induction hypothesis follows that $\S{A}(p_{k-1}) \neq \S{A}(p_k) = \S{A}(\tilde{p}_k)$. This shows that $\Pi(x) = p_1 \cdots p_{k-1} \tilde{p}_k$ is a prime factorization of $x$. 
\item 
Case $a \neq x_{n+1}$: The word $p_{k+1} = x_{n+1}$ is a prime factor of $x$ satisfying $p_1 \cdots p_k p_{k+1} = x$ and 
$\S{A}(p_k) = a \neq x_{n+1} = \S{A}(p_{k+1})$. Then from the induction hypothesis follows that $\Pi(x) = p_1 \cdots p_{k+1}$ is a prime factorization of $x$.
\end{enumerate}

\medskip

It remains to show that a prime factorization is unique. Let $\Pi(x) = p_1 \cdots p_{k}$ and $\Pi'(x) = q_1 \cdots q_l$ be different prime factorizations of $x$. Suppose that the prime factors are of the form $p_r = a_r^{m_r}$ and $q_s = b_s^{n_s}$ for all $r \in [k]$ and $s \in [l]$, where $a_r, b_s \in \S{A}$ and $m_r, n_s \in \N$. The partition-condition of a prime factorization yields
\[
x = a_1^{m_1} \cdots a_k^{m_k} = b_1^{n_1} \cdots b_l^{n_l}.
\]
Since $\Pi(x)$ and $\Pi'(x)$ are different, we can find a smallest index $i \in [k] \cap [l]$ such that $p_i \neq q_i$. 
From $p_i \neq q_i$ follows that $a_i^{m_i} \neq b_i^{n_i}$. Since $i$ is the smallest index for which the prime factors of both prime factorizations differ, we have
\[
u = a_1^{m_1} \cdots a_{i-1}^{m_{i-1}} = b_1^{n_1} \cdots b_{i-1}^{n_{i-1}}.
\]
There are words $v, w \in \S{A}^*$ such that $x = u a_i v = u b_i w = x$. This shows that $a_i = b_i$. Hence, from $a_i^{m_i} \neq b_i^{n_i}$ and $a_i = b_i$ follows that $m_i \neq n_i$. We distinguish between two cases: (i) $m_i < n_i$ and (ii) $m_i > n_i$. We assume that $m_i < n_i$. Suppose that
\[
u = a_1^{m_1} \cdots a_{i-1}^{m_{i-1}}a_i^{m_i} = b_1^{n_1} \cdots b_{i-1}^{n_{i-1}}b_i^{m_i}.
\]
From $m_i < n_i$ follows $\abs{u} < \abs{x}$. This implies that $i < k$, that is $p_{i+1}$ is an element of $\Pi(x)$. Thus, there are words $v, w \in \S{A}^*$ such that $x = u a_{i+1} v = u b_i w$. We obtain 
\[
\S{A}(p_{i+1}) = a_{i+1} = b_i = a_i = \S{A}(p_i). 
\]
This violates the maximality-condition of a prime factorization and therefore contradicts our assumption that $m_i < n_i$. In a similar way, we can derive a contradiction for the second case $m_i > n_i$. Combining the results of both cases gives $m_i = n_i$, which contradicts our assumption that both prime factorizations are different. Thus a prime factorization is uniquely determined. This completes the proof. 
\end{proof}

\medskip

Let $x$ and $x'$ be words over $\S{A}$ whose prime factorizations $\Pi(x) = p_1 \cdots p_k$ and $\Pi(y) = p'_1 \cdots p'_k$, resp., have the same length $k$. We say $x$ is an \emph{expansion} of $x'$, written as $x \succ x'$, if $\S{A}(p_t) = \S{A}(p'_t)$ and $\abs{p_t} \geq \abs{p'_t}$ for all $t \in [k]$. If $x$ is an expansion of $x'$, then $x'$ is called a \emph{compression} of $x$, denoted by $x' \prec x$. By definition, every word is an expansion (compression) of itself. The set 
\[
\S{C}(x) = \cbrace{x' \in \S{A}^* \,:\, x' \prec x}
\]
is the set of all compressions of $x$. Suppose that $x_0, x_1, \ldots, x_k$ are $k+1$ time series. Occasionally, we write $x_0 \succ x_1, \ldots, x_k$ for $x_0 \succ x_1, \ldots, x_0 \succ x_k$ and similarly $x_0 \prec x_1, \ldots, x_k$ for $x_0 \prec x_1, \ldots, x_0 \prec x_k$.

An \emph{irreducible word} is a word $x = x_1 \cdots x_n$ over $\S{A}$ with prime factorization $\Pi(x) = x_1 \cdots x_n$. From $\Pi(\varepsilon) = \varepsilon$ follows that the empty word $\varepsilon$ is irreducible. A \emph{condensed form} of $x$ a word over $\S{A}$ is an irreducible word $x^*$ such that $x \succ x^*$. The condensed form of $\varepsilon$ is $\varepsilon$. We show that every word has a uniquely determined condensed form.

\begin{proposition}\label{prop:existence-of-condensed-form}
Every word has a unique condensed form. 
\end{proposition}

\begin{proof}
It is sufficient to show existence and uniqueness of a condensed form of non-empty words. Let $x$ be a word over $\S{A}$ with prime factorization $\Pi(x) = p_1 \cdots p_k$. Let $\S{A}(p_t) = \cbrace{a_t}$ for all $t \in [k]$. We define the word $x^* = a_1 \cdots a_k$. Then $\Pi(x^*) = a_1 \cdots a_k$ is the prime factorization of $x^*$ that obviously satisfies the partition-condition and whose maximality-condition is inherited by the maximality-condition of the prime factorization $\Pi(x)$. This shows that $x^*$ is irreducible and a compression of $x$. Hence, $x^*$ is a condensed form of $x$. 

Suppose that $z = b_1 \cdots b_l$ is an irreducible word such that $z \prec x$. Then we have $k = l$, $\S{A}(p_t) = \S{A}(b_t)$, and $\abs{p_t} \geq \abs{b_t} = 1$ for all $t \in [k]$. Observe that $\cbrace{a_t} = \S{A}(p_t) = \S{A}(b_t) = \cbrace{b_t}$. We obtain $z = a_1 \cdots a_k$ showing the uniqueness of the condensed form of $x$. 
\end{proof}

\medskip

The next result shows that expansions (compressions) are transitive. 
\begin{proposition}\label{prop:expansions-are-transitive}
Let $x, y, z \in \S{A}^*$. From $x \succ y$ and $y \succ z$ follows $x \succ z$. 
\end{proposition}

\begin{proof}
Let $\Pi(x) = u_1 \cdots u_k$, $\Pi(y) = v_1 \cdots v_l$, and $\Pi(z) = w_1 \cdots w_m$ be the prime factorizations of $x$, $y$, and $z$, respectively. From $x \succ y$ and $y \succ z$ follows that $k = l = m$. Let $t \in [k]$. We have $u_t \succ v_t$ and $v_t \succ w_t$. From $u_t \succ v_t$ follows that there are elements $a_t \in \S{A}$ and $\alpha_t, \beta_t \in \N$ with $\alpha_t \geq \beta_t$ such that $u_t = a_t^{\alpha_t}$ and $v_t = a_t^{\beta_t}$. From $v_t \succ w_t$ follows that there is an element $\gamma_t \in \N$ with $\beta_t \geq \gamma_t$ such that $w_t = a_t^{\gamma_t}$. Since $\alpha_t \geq \beta_t$ and $\beta_t \geq \gamma_t$, we obtain $u_t \succ w_t$. We have chosen $t \in [k]$ arbitrarily. Hence, we find that $x \succ z$. This proves the assertion.
\end{proof}

\medskip

Expansions have been introduced as expansions on the prime factors. Lemma \ref{lemma:expansions} states that expansions are obtained by replicating a subset of elements of a given word.

\begin{lemma}\label{lemma:expansions}
Let $x = x_1 \cdots x_n$ and $y = y_1 \cdots y_m$ be words over $\S{A}$. Then the following statements are equivalent:
\begin{enumerate}
\itemsep0em
\item
$x$ is an expansion of $y$.
\item 
There are $\alpha_1, \ldots, \alpha_m \in \N$ such that $x = y_1^{\alpha_1} \cdots y_m^{\alpha_m}$.
\end{enumerate}
\end{lemma}

\begin{proof} 
Suppose that $\Pi(x) = p_1 \cdots p_k$ and $\Pi(y) = q_1 \cdots q_l$ are the prime factorization of $x$ and $y$, respectively. 

\medskip

\noindent
$\Rightarrow$: We assume that $x$ is an expansion of $y$. Then from $x \succ y$ follows that $k = l$, $\S{A}(p_t) = \S{A}(q_t)$, and $\abs{p_t} \geq \abs{q_t}$ for all $t \in [k]$. We arbitrarily pick an element $t \in [k]$. There are elements $a_t \in \S{A}$ and $\beta_t, \gamma_t \in \N$ such that $p_t = a_t^{\beta_t}$, $q_t = a_t^{\gamma_t}$ and $\beta_t \geq \gamma_t$. In addition, there is an index $i_t \in [m]$ such that 
\begin{align*}
q_t = y_{i_t} y_{i_t+1}\cdots y_{i_t+\gamma_t-1} = a_t^{\gamma_t}.
\end{align*}
Let $\nu_t = \beta_t - \gamma_t -1$. Then $\nu_t \geq 0$ and we have
\[
p_t = y_{i_t}^{\nu_t} y_{i_t+1}^1 \cdots y_{i_t+\gamma_t-1}^{1} = a_t^{\beta_t}.
\]
We set $\alpha_{i_t} = \nu_t$ and $\alpha_{i_t+1} = \cdots \alpha_{i_t+\gamma_t-1} = 1$. Concatenating all prime factors $p_t$ of $x$ yields the assertion. 

\medskip

\noindent
$\Leftarrow$: We assume that there are integers $\alpha_1, \ldots, \alpha_m \in \N$ such that $x = y_1^{\alpha_1} \cdots y_m^{\alpha_m}$. Suppose that $\S{A}(q_t) = \cbrace{a_t}$ for all $t \in [l]$. Then there are integers $i_0, \ldots, i_l \in \N$ such that $0 = i_0 < i_1 < \cdots < i_l = m$ and 
\[
a_t = y_{i_{t-1}+1} = \cdots = y_{i_t}
\] 
for all $t \in [l]$. We set $\beta_t = i_t - (i_{t-1}+1)$ and $\gamma_t = \alpha_{i_{t-1}+1} + \cdots + \alpha_{i_t}$ for all $t \in [l]$. From $1 \leq \alpha_i$ for all $i \in [m]$ follows $\beta_t \leq \gamma_t$ for all $t \in [l]$. Hence, we have 
\begin{align*}
y = a_1^{\beta_1} \cdots a_l^{\beta_l} \quad \text{ and } \quad 
x = a_1^{\gamma_1} \cdots a_l^{\gamma_l}.
\end{align*}
Hence, $k = l$, $\S{A}(p_t) = \S{A}(q_t)$, and $\abs{p_t} = \gamma_t \geq \beta_t=\abs{q_t}$ for all $t \in [l]$. This shows $x \succ y$. 
\end{proof}

\medskip

The next result shows that the set of compressions of an irreducible word is a singleton. 

\begin{proposition}
Let $x \in \S{A}^*$ be irreducible. Then $\S{C}(x) = \cbrace{x}$. 
\end{proposition}

\begin{proof}
By definition, we have $x \in \S{C}(x)$. Suppose there is a word $y \in \S{C}(x)$ with prime factorization $\Pi(y) = p_1 \cdots p_k$. Let $y^* = a_1 \cdots a_k$ be the condensed form of $y$, where $a_t \in \S{A}(p_t)$ for all $t \in [k]$. From $y \prec x$ and $y^* \prec y$ follows $y^* \prec x$ by Prop.~\ref{prop:expansions-are-transitive}. According to Lemma \ref{lemma:expansions} there are positive integers $\alpha_1, \ldots, \alpha_k \in \N$ such that $x = a_1^{\alpha_1} \cdots a_k^{\alpha_k}$. Since $x$ is irreducible all $\alpha_t$ have value one giving $x = a_1 \cdots a_k$. Hence, we have $x = y^*$. In addition, from $x \succ y$ and $y \succ x$ follows $x = y$. This shows the assertion.
\end{proof}

\medskip

Proposition \ref{prop:compression-is-expansion-of-condensed-form} states that every compression of a word is an expansion of its condensed form. 

\begin{proposition}\label{prop:compression-is-expansion-of-condensed-form}
Let $x$ be a word over $\S{A}$ with condensed form $x^*$. Suppose that $y\in \S{A}^*$ such that $y \prec x$. Then $x^* \prec y$. 
\end{proposition}

\begin{proof}
Let $x \in \S{A}^*$ be a word with prime factorization $\Pi(x) = p_1 \cdots p_k$. Suppose that $y \in \S{C}(x)$ with condensed form 
$\Pi(y) = q_1 \cdots q_l$. From $y \prec x$ follows that $k = l$ and $q_t \prec p_t$ for all $t \in [k]$. This implies that $x$ and $y$ have the same condensed form $x^*$. Hence, we have $x^* \prec y$, which completes the proof.
\end{proof}

\medskip

Suppose that $\S{C}(x)$ is the set of compressions of a word $x$. We show that the shortest word in $\S{C}(x)$ is the condensed form of $x$. 

\begin{proposition}\label{prop:minimum-length-of-condensed-form}
Let $x$ be a word over $\S{A}$ with condensed form $x^*$. Then $\abs{x^*} < \abs{y}$ for all $y \in \S{C}(x) \setminus \cbrace{x^*}$.
\end{proposition}

\begin{proof}
Let $x^* = a_1 \cdots a_k$ and let $y \in \S{C}(x)$. From Prop.~\ref{prop:compression-is-expansion-of-condensed-form} follows that $x^* \prec y$. Lemma \ref{lemma:expansions} gives positive integers $\alpha_1, \ldots \alpha_k \in \N$ such that $y = a_1^{\alpha_1} \cdots a_k^{\alpha_k}$. This shows that $\abs{y} = \alpha_1 + \cdots + \alpha_k \geq k = \abs{x^*}$. Suppose that $\abs{y} = k$. In this case, we have $\alpha_1 = \cdots = \alpha_k = 1$ and therefore $y = x^*$. This shows that $\abs{x^*} < \abs{y}$ for all $y \in \S{C}(x)\setminus \cbrace{x^*}$.
\end{proof}

Suppose that $x, y, z$ are words over $\S{A}^*$. We say, $z$ is a \emph{common compression} of $x$ and $y$ if $z \prec x, y$. A \emph{compression-expansion} (co-ex) function is a function $f: \S{A}^* \rightarrow \S{A}^*$ such that there is a common compression of $x$ and $f(x)$. Proposition \ref{prop:composition-of-co-ex-functions} states that co-ex functions are closed under compositions. 

\begin{proposition}\label{prop:composition-of-co-ex-functions}
The composition of two co-ex functions is a co-ex function.
\end{proposition}

\begin{proof}
To prove the assertion, we repeatedly apply transitivity of expansions (Prop.~\ref{prop:expansions-are-transitive}). 
Let $x \in \S{A}^*$ and let $f = g \circ h$ be the composition of two co-ex functions $g, h: \S{A}^* \rightarrow \S{A}^*$. Then there are words $z_h$ and $z_g$ such that $z_h \prec x, h(x)$ and $z_g \prec h(x), g(h(x))$. Suppose that $z_h^*$ and $z_g^*$ are the condensed forms of $z_h$ and $z_g$, respectively. From 
\begin{align*}
z_h^* \prec z_h \prec h(x)
\quad \text{ and } \quad
z_g^* \prec z_g \prec h(x)
\end{align*}
follows $z_h^* \prec h(x)$ and $z_g^* \prec h(x)$ by the transitivity of expansions. According to Prop.~\ref{prop:existence-of-condensed-form}, the condensed form of a word is unique. Hence, we have $z_h^* = z_g^*$. We set $x^* = z_h^* = z_g^*$. Then from 
\begin{align*}
x^* = z_h^* \prec z_h \prec x
\quad \text{ and } \quad
x^* = z_g^*\prec z_g \prec g(h(x)) = f(x)
\end{align*}
follows $x^* \prec x, f(x)$ by the transitivity of expansions. This shows that $x^*$ is a common compression of $x$ and $f(x)$. Since $x$ was chosen arbitrarily, the assertion follows. 
\end{proof}

\subsection{Warping Walks}\label{subsec:warping-walks}

The standard definition of the dtw-distance is inconvenient for our purposes. The recursive definition of warping paths is easy to understand and well-suited for deriving algorithmic solutions, but often less suited for a theoretical analysis. In addition, warping paths are not closed under compositions. As a more convenient definition, we introduce warping walks. Warping walks generalize warping paths by slightly relaxing the step condition. Using warping functions and matrices, this section shows that warping walks do not affect the dtw-distance. The next section shows that warping walks are closed under compositions.

\medskip

\commentout{
\begin{notation}
Let $\B = \cbrace{0,1}$ and let $I_n \in \R^{n \times n}$ be the identity matrix. 
Let $A = (a_{ij}) \in \R^{m \times n}$ be a matrix. Then we write $A_i = (a_{i1}, \ldots, a_{in})$ for the $i$-th row and $A^j = (a_{1j}, \ldots, a_{mj})$ for the $j$-th column of $A$.
\QED
\end{notation}
}

\begin{notation}
Let $\B = \cbrace{0,1}$ and let $I_n \in \R^{n \times n}$ be the identity matrix. 
\QED
\end{notation}

\medskip

Let $\ell, n \in \N$. A function $\phi:[\ell] \rightarrow [n]$ is a \emph{warping function} if it is surjective and monotonically increasing. Thus, for a warping function we always have $\ell \geq n$. The \emph{warping matrix} associated with warping function $\phi$ is a matrix of the form
\[
\Phi = \begin{pmatrix}
e_{\phi(1)}\\
\vdots\\
e_{\phi(\ell)}
\end{pmatrix} \in \B^{\ell \times n},
\]
where $e_i$ is the $i$-th standard basis vector of $\R^n$, denoted as a row vector, with $1$ in the $i$-th position and $0$ in every other position. The next result shows the effect of multiplying a time series with a warping matrix. 

\begin{lemma}\label{lemma:expansion-by-phi}
Let $\phi:[\ell] \rightarrow [n]$ be a warping function with associated warping matrix $\Phi$. Suppose that $x = (x_1, \ldots, x_n)s \in \S{T}$ is a time series of length $\abs{x} = n$. Then there are elements $\alpha_1, \ldots, \alpha_n \in \N$ such that 
\[
\Phi x = (\underbrace{x_1, \ldots, x_1}_{\alpha_1-\text{times}}, \underbrace{x_2, \ldots, x_2}_{\alpha_2-\text{times}},\ldots \underbrace{x_n, \ldots, x_n}_{\alpha_n-\text{times}})^\intercal.
\] 
\end{lemma}

\begin{proof}
Since $\phi$ is surjective and monotonic, we can find integers $\alpha_1, \ldots, \alpha_n \in \N$ such that 
\[
\phi(i) = \begin{cases}
1 & i \leq \alpha_1\\
2 & \alpha_1 < i \leq \alpha_1 + \alpha_2\\
\cdots & \cdots \\
n & \alpha_1 + \cdots + \alpha_{n-1}< i 
\end{cases}
\]
for all $i \in [\ell]$. Let $\Phi \in \B^{\ell \times n}$ be the warping matrix associated with $\phi$. Then the $n$ rows $\Phi_i$ of $\Phi$ are of the form
\begin{align*}
\Phi_1 &= \cdots = \Phi_{\alpha_1} = e_1\\
\Phi_{\alpha_1+1} &= \cdots = \Phi_{\alpha_2} = e_2\\
&\;\;\vdots \\
\Phi_{\alpha_{m-1}+1} &= \cdots = \Phi_{\alpha_m} = e_m,
\end{align*}
Obviously, the warping matrix $\Phi$ satisfies 
\[
\Phi x = (\underbrace{x_1, \ldots, x_1}_{\alpha_1-\text{times}}, \underbrace{x_2, \ldots, x_2}_{\alpha_2-\text{times}},\ldots \underbrace{x_n, \ldots, x_n}_{\alpha_n-\text{times}})^\intercal.
\] 
\end{proof}

\medskip

A \emph{warping walk} is a pair $w = (\phi, \psi)$ consisting of warping functions $\phi:[\ell] \rightarrow [m]$ and $\psi:[\ell] \rightarrow [n]$ of the same domain $[l]$. The warping walk $w$ has \emph{order} $m \times n$ and \emph{length} $\ell$. By $\S{W}_{m,n}$ we denote the set of all warping walks of order $m \times n$ and of finite length.

In the classical terminology of dynamic time warping, a warping walk can be equivalently expressed by a sequence $w = (w_1, \ldots, w_{\ell})$ of $\ell$ points $w_l = (\phi(l), \psi(l)) \in [m] \times [n]$ such that the following conditions are satisfied:
\begin{enumerate}
\item $w_1 = (1,1)$ and $w_\ell = (m,n)$ \hfill (\emph{boundary condition})
\item $w_{l+1} - w_{l} \in \B \times \B$ for all $l \in [\ell-1]$ \hfill(\emph{weak step condition})
\end{enumerate}
The weak step condition relaxes the standard step condition of warping paths by additionally allowing zero-steps of the form $w_l - w_{l+1} = (0,0)$. Zero-steps duplicate points $w_l$ and thereby admit multiple correspondences between the same elements of the underlying time series. 

\begin{notation}
We identify warping walks $(\phi, \psi)$ with their associated warping matrices $(\Phi, \Psi)$ and interchangeably write $(\phi, \psi) \in \S{W}_{m,n}$ and $(\Phi, \Psi) \in \S{W}_{m,n}$.
\end{notation}

\medskip

A warping walk $p = (p_1, \ldots, p_{\ell})$ is a \emph{warping path} if $p_{l+1} \neq p_l$ for all $l \in [\ell-1]$. By $\S{P}_{m,n}$ we denote the subset of all warping paths of order $m \times n$. Thus, a warping path is a warping walk without consecutive duplicates. Equivalently, a warping path satisfies the boundary condition and the strict step condition
\begin{enumerate}
\item[$2'.$] $w_{l+1} - w_{l} \in \B \times \B \setminus \cbrace{0,0}$ for all $l \in [\ell-1]$ \hfill(\emph{strict step condition})
\end{enumerate}

Warping walks are words over the alphabet $\S{A} = \N \times \N$ and warping paths are irreducible words over $\S{A}$. For the sake of convenience, we regard $\S{W}_{m,n}$ and $\S{P}_{m,n}$ as subsets of $\S{A}^*$. The \emph{condensation map}
\[
c: \S{A}^* \rightarrow \S{A}^*, \quad w \mapsto w^*
\]
sends a word $w$ over $\S{A}$ to its condensed form $w^*$. 
\begin{lemma}\label{lemma:c(W)=P}
Let $\S{A} = \N \times \N$ and let $c: \S{A}^* \rightarrow \S{A}^*$ be the condensation map. Then $c(\S{W}_{m,n}) = \S{P}_{m,n}$ for all $m,n \in \N$. 
\end{lemma}

\begin{proof}
Let $m,n \in \N$. 

\medskip

\noindent
$\S{P}_{m,n}\subseteq c(\S{W}_{m,n})$: Let $p = (p_1, \ldots, p_\ell) \in \S{P}_{m,n}$ be a warping path. From the strict step condition follows that $p$ is irreducible. Consider the word $w = (p_1, \ldots, p_\ell, p_{\ell+1})$, where $p_\ell = p_{\ell+1}$. The word $w$ satisfies the boundary and weak step condition. Hence, $w$ is a warping walk with unique prime factorization $\Pi(w) = (p_1 \cdots p_\ell)$. This shows that $p$ is the unique condensed form of $w$. Hence, we have $\S{P}_{m,n}\subseteq c(\S{W}_{m,n})$.

\medskip

\noindent
$c(\S{W}_{m,n})\subseteq \S{P}_{m,n}$: A warping walk $w \in \S{W}_{m,n}$ satisfies the boundary and the weak step condition. As an irreducible word, the condensed form $w^* = c(w)$ satisfies the boundary and the strict step condition. Hence, $w^*$ is a warping path. This proves $c(\S{W}_{m,n})\subseteq \S{P}_{m,n}$.
\end{proof}

\medskip

The \emph{dtw-distance} is a distance function on $\S{T}$ of the form
\[
\dtw: \S{T} \times \S{T} \rightarrow \R_{\geq 0}, \quad (x, y) \mapsto \min \cbrace{\norm{\Phi x - \Psi y} \,:\, (\Phi, \Psi) \in \S{P}_{\abs{x},\abs{y}}}
\]
From \cite{Schultz2018}, Prop.~A.2 follows that the dtw-distance coincides with the standard definition of the dtw-distance as presented in Section \ref{sec:results}. The next result expresses the dtw-distance in terms of warping walks.

\begin{proposition}\label{prop:dtw}
Let $x, y \in \S{T}$ be two time series. Then we have
\begin{align*}
\dtw(x, y) = \min \cbrace{\norm{\Phi x - \Psi y} \,:\, (\Phi, \Psi) \in \S{W}_{\abs{x},\abs{y}}}.
\end{align*}
\end{proposition}

\begin{proof}
Let $w \in \S{W}_{m,n}$ be a warping walk. Then $p=c(w)$ is a warping path and a condensed form of $w$ by Lemma \ref{lemma:c(W)=P}. From Prop.~\ref{prop:minimum-length-of-condensed-form} follows that $\abs{p} \leq \abs{w}$. Thus, we obtain
\[
C_p(x, y) = \sum_{(i,j) \in p} (x_i-y_j)^2 \leq \sum_{(i,j) \in w} (x_i-y_j)^2 = C_w(x,y).
\] 
This implies
\[
\dtw(x, y) \leq \min \cbrace{\norm{\Phi x - \Psi y} \,:\, (\Phi, \Psi) \in \S{W}_{m,n} \setminus \S{P}_{m,n}}
\]
and proves the assertion.
\end{proof}

\medskip

We call a warping walk $(\Phi, \Psi)$ \emph{optimal} if $\norm{\Phi x - \Psi y} = \dtw(x, y)$. From Prop.~\ref{prop:dtw} follows that transition from warping paths to warping walks leaves the dtw-distance unaltered and that we can condense every optimal warping walk to an optimal warping path by removing consecutive duplicates.

\subsection{Properties of Warping Functions}\label{subsec:properties-of-warping-walks}

In this section, we compile results on compositions of warping functions and warping walks. We begin with showing that warping functions are closed under compositions. 
\begin{lemma}\label{lemma:composition-01}
Let $\phi:[\ell] \rightarrow [m]$ and $\psi:[m] \rightarrow [n]$ be warping functions. Then the composition
\[
\theta: [\ell] \rightarrow [n], \quad l \mapsto \psi(\phi(l))
\]
is also a warping function. 
\end{lemma}

\begin{proof}
The assertion follows, because the composition of surjective (monotonic) functions is surjective (monotonic). 
\end{proof}

\medskip

The composition of warping functions is contravariant to the composition of their associated warping matrices. Suppose that $\Phi \in \B^{\ell \times m}$ and $\Psi \in \B^{m \times n}$ are the warping matrices of the warping functions $\phi$ and $\psi$ from Lemma \ref{lemma:composition-01}, respectively. Then the warping matrix of the composition $\theta = \psi \circ \phi$ is the matrix product $\Phi\Psi \in \B^{\ell \times n}$. The next result shows that warping walks are closed for a special form of compositions. 

\begin{lemma}\label{lemma:composition-02}
Let $(\phi, \psi) \in \S{W}_{m,n}$ be a warping walk and let $\theta: [m] \rightarrow [r]$ be a warping function. Then $(\theta \circ \phi, \psi)$ is a warping walk in $\S{W}_{r,n}$. 
\end{lemma}

\begin{proof}
Follows from Lemma \ref{lemma:composition-01} and by definition of a warping walk.
\end{proof}

\medskip

Let $\phi:[m] \rightarrow [n]$ and $\phi':[m'] \rightarrow [n]$ be warping functions. The \emph{pullback} of $\phi$ and $\phi'$ is the set of the form
\[
\phi \otimes \phi' = \cbrace{(u,u') \in [m] \times [m'] \,:\, \phi(u) = \phi'(u')}.
\]
By $\pi: \phi \otimes \phi' \rightarrow [m]$ and $\pi': \phi \otimes \phi' \rightarrow [m']$ we denote the canonical projections. Let 
$\psi = \phi \circ \pi$ and $\psi' = \phi' \circ \pi'$ be the compositions that send elements from the pullback $\phi \otimes \phi'$ to the set $[n]$. The \emph{fiber} of $i \in [n]$ under the map $\psi$ is the set $\S{F}(i) = \cbrace{(u,u') \in \phi \otimes \phi' \,:\, \psi(u,u') = i}$. In a similar way, we can define the fiber of $i$ under the map $\psi'$. The next results show some properties of pullbacks and their fibers. 

\begin{lemma}\label{lemma:psi=psi'}
Let $\phi \otimes \phi'$ be a pullback of warping functions $\phi$ and $\phi'$. Then the compositions $\psi = \phi \circ \pi$ and $\psi' = \phi' \circ \pi'$ are surjective and satisfy $\psi = \psi'$.
\end{lemma}

\begin{proof}
Warping functions and the natural projections are surjective. As a composition of surjective functions, the functions $\psi$ and $\psi'$ are surjective. For every $(u,u') \in \phi \otimes \phi'$ we have
\[
\psi(u,u') = \phi(\pi(u,u')) = \phi(u) = \phi'(u') = \phi'(\pi'(u,u')) = \psi'(u,u').
\]
This proves the assertion $\psi = \psi'$. 
\end{proof}

\medskip

Lemma \ref{lemma:psi=psi'} has the following implications: First, from $\psi = \psi'$ follows that the fiber of $i$ under the map $\psi$ coincides with the fiber of $i$ under the map $\psi'$. Second, since $\psi$ is surjective, the fibers $\S{F}(i)$ are non-empty for every $i \in [n]$. Third, the fibers $\S{F}(i)$ form a partition of the pullback $\phi \otimes \phi'$.

\begin{lemma}\label{lemma:fibers-are-order-preserving}
Let $\phi \otimes \phi'$ be a pullback of warping functions $\phi:[m] \rightarrow [n]$ and $\phi':[m'] \rightarrow [n]$. Suppose that $i, j \in [n]$ with $i < j$. From $(u_i, u'_i) \in \S{F}(i)$ and $(u_j, u'_j) \in \S{F}(j)$ follows $u_i \leq u_j$ and $u'_i \leq u'_j$.
\end{lemma}
\begin{proof}
We show the first assertion $u_i \leq u_j$. The proof for the second assertion $u'_i \leq u'_j$ is analogous. Suppose that $u_i > u_j$. From $(u_i, u'_i) \in \S{F}(i)$ follows $\phi(u_i) = i$ and from $(u_j, u'_j) \in \S{F}(j)$ follows $\phi(u_j) = j$. Since $\phi$ is monotonic, we have $i = \phi(u_i) \geq \phi(u_j) = j$, which contradicts the assumption that $i < j$. This shows $u_i \leq u_j$. 
\end{proof}

\begin{lemma}\label{lemma:representation-of-fibers}
Let $\phi \otimes \phi'$ be a pullback of warping functions $\phi:[m] \rightarrow [n]$ and $\phi':[m'] \rightarrow [n]$. For every $i \in [n]$ there are elements $u_i \in [m]$, $u'_i \in [m']$ and $k_i, l_i \in \N$ such that 
\begin{align*}
\pi(\S{F}(i)) = \cbrace{u_i, u_i + 1, \ldots, u_i + k_i} \quad \text{ and } \quad
\pi'(\S{F}(i)) = \cbrace{u'_i, u'_i + 1, \ldots, u'_i + l_i}.
\end{align*}
\end{lemma}
\begin{proof}
We show the assertion for $\pi(\S{F}(i))$. The proof of the assertion for $\pi'(\S{F}(i))$ is analogous. Let $i \in [n]$ and let $\S{G}(i) = \pi(\S{F}(i))$. Since fibers are non-empty and finite, we can find elements $u_i \in [m]$ and $k_i \in \N$ such that 
$\S{G}(i) = \cbrace{u_i, u_{i+1}, \ldots, u_{i + k_i}}$ with $u_i < u_{i+1} < \cdots < u_{i + k_i}$. It remains to show that $u_{i+r} = u_i + r$ for all $r \in [k_i]$.

We assume that there is a smallest number $r \in [k_i]$ such that $u_{i+r} \neq u_i + r$. Then $r \geq 1$ and therefore $i+r-1 \geq i$. This shows that $u_{i+r-1} \in \S{G}(i)$. Observe that $u_{i+r-1} = u_i + r-1$, because $r$ is the smallest number violating $u_{i+r} = u_i + r$. From $u_{i+r-1} < u_{i+r}$ together with $u_i + r \notin \S{G}(i)$ follows
\[
u_{i+r-1} = u_i + r-1 < u_i + r < u_{i+r}.
\]
Recall that the fibers form a partition of the pullback $\phi \otimes \phi'$. Then there is a $j \in [n] \setminus \cbrace{i}$ such that $u_i+r \in \S{G}(j)$. We distinguish between two cases:\footnote{The case $i = j$ can not occur by assumption.} 

\medskip

\noindent
Case $i < j$: From Lemma \ref{lemma:fibers-are-order-preserving} follows that $u_{i+r} \leq u_i+r$, which contradicts the previously derived inequality $u_i + r < u_{i+r}$.

\medskip

\noindent
Case $j < i$: From Lemma \ref{lemma:fibers-are-order-preserving} follows that $u_i+r \leq u_i$. Observe that either $u_i = u_{i+r-1}$ or $u_i < u_{i+r-1}$. We obtain the contradiction $u_i \leq u_{i+r-1} < u_i + r \leq u_i$. 

\medskip

From both contradictions follows that $u_{i+r} = u_i + r$ for every $r \in [k_i]$. This completes the proof. 
\end{proof}

\medskip

Lemma \ref{lemma:pullback} uses pullbacks to show that pairs of warping functions with the same co-domain can be equalized by composition with a suitable warping walk. 

\begin{lemma}\label{lemma:pullback}
Let $\phi:[m] \rightarrow [n]$ and $\phi':[m'] \rightarrow [n]$ be warping functions. Then there are warping functions $\theta: [\ell] \rightarrow [m]$ and $\theta': [\ell] \rightarrow [m']$ for some $\ell \geq \max(m, m')$ such that $\phi \circ \theta = \phi' \circ \theta'$. 
\end{lemma}

\begin{proof}
We first construct a suitable set $\S{Z}$ of cardinality $\abs{Z} = \ell$. For this, let $\phi \otimes \phi'$ be the pullback of $\phi$ and $\phi'$ and let $i \in [n]$. From Lemma \ref{lemma:representation-of-fibers} follows that there are elements $u_i \in [m]$, $u'_i \in [m']$ and $k_i, l_i \in \N$ such that 
\begin{align*}
\pi(\S{F}(i)) = \cbrace{u_i, u_{i + 1}, \ldots, u_{i + k_i}} 
\qquad \text{and} \qquad
\pi'(\S{F}(i)) = \cbrace{u'_i, u'_{i + 1}, \ldots, u'_{i + l_i}}.
\end{align*}
where $u_{i+r} = u_i + r$ for all $r \in [k_i]$ and $u'_{i+s} = u'_i + s$ for all $s \in [l_i]$. Without loss of generality we assume that $k_i \leq l_i$. For every $i \in [n]$, we construct the subset
\[
\S{Z}(i) = \cbrace{(u_i, u'_i), (u_{i+1}, u'_{i+1}), \ldots, (u_{i + k_i}, u'_{i+k_i}), (u_{i + k_i}, u'_{i+k_i+1})\ldots, (u_{i + k_i}, u'_{i+l_i})} \subseteq \S{F}(i)
\]
and form their disjoint union
\[
\S{Z} = \bigcup_{i \in [n]} \S{Z}(i) \subseteq \phi \otimes \phi'.
\]
Let $\leq_{\S{Z}}$ be the lexicographical order on $\S{Z}$ defined by
\[
(u_i,u'_i) \leq_{\S{Z}} (u_j,u'_j) \quad \text{ if and only if } \quad (u_i < u_j) \text{ or } (u_i = u_j \text{ and } u'_i \leq u'_j).
\]
for all $(u_i,u'_i), (u_j,u'_j) \in \S{Z}$. We show that the properties of $\S{Z}$ yield a tighter condition on $\leq_{\S{Z}}$. Let $(u_i,u'_i), (u_j,u'_j) \in \S{Z}$ such that $(u_i,u'_i) \leq_{\S{Z}} (u_j,u'_j)$. Then there are $i,j \in [n]$ such that $(u_i,u'_i) \in \S{Z}(i)$ and $(u_j,u'_j) \in \S{Z}(j)$. We distinguish between three cases:
\begin{enumerate}
\itemsep0em
\item $i = j$: By construction of $\S{Z}(i)$, the relationship $u_i \leq u_j$ gives $u'_i \leq u'_j$.
\item $i < j$: From Lemma \ref{lemma:fibers-are-order-preserving} follows that $u_i \leq u_j$ and $u'_i \leq u'_j$. 
\item $i > j$: Lemma \ref{lemma:fibers-are-order-preserving} yields $u_i \geq u_j$ and $u'_i \geq u'_j$. The assumption $(u_i,u'_i) \leq_{\S{Z}} (u_j,u'_j)$ gives $u_i \leq u_j$. Then from $u_i \geq u_j$ and $u_i \leq u_j$ follows $u_i = u_j$. In addition, we have $u'_i \leq u'_j$ by $(u_i,u'_i) \leq_{\S{Z}} (u_j,u'_j)$ and $u_i = u_j$. Hence, from $u'_i \geq u'_j$ and $u'_i \leq u'_j$ follows $u'_i = u'_j$. In summary, we have $u_i = u_j$ and $u'_i = u'_j$.
\end{enumerate}
The case distinction yields 
\[
(u_i,u'_i) \leq_{\S{Z}} (u_j,u'_j) \quad \text{ if and only if } \quad (u_i \leq u_j) \text{ and } (u'_i \leq u'_j).
\]
for all $(u_i,u'_i), (u_j,u'_j) \in \S{Z}$. Suppose that $\ell = \abs{\S{Z}}$. Then there is a bijective function
\[
f: [\ell] \rightarrow \S{Z}, \quad i \mapsto f(i)
\]
such that $i \leq j$ implies $f(i) \leq_{\S{Z}} f(j)$ for all $i,j \in [\ell]$. 

Next, we show that the functions $\theta = \pi \circ f$ and $\theta' = \pi' \circ f$ are warping functions on $[\ell]$. Both functions $\theta$ and $\theta'$ are surjective as compositions of surjective functions. To show that $\theta$ and $\theta'$ are monotonic, we assume that $i, j \in [\ell]$ with $i < j$. Suppose that $f(i) = (u_i,u'_i)$ and $f(j) = (u_j,u'_j)$. From $i < j$ follows $(u_i,u'_i) \leq_{\S{Z}} (u_j,u'_j)$ by construction of $f$. From $(u_i,u'_i) \leq_{\S{Z}} (u_j,u'_j)$ follows $u_i \leq u_j$ and $u'_i \leq u'_j$ as shown in the first part of this proof. Hence, we find that 
\begin{align*}
\theta(i) = \pi(f(i)) = \pi(u_i, u'_i) = u_i &\leq u_j = \pi(u_j, u'_j) = \pi(f(j)) = \theta(j)\\
\theta'(i) = \pi'(f(i)) = \pi'(u_i, u'_i) = u'_i &\leq u'_j = \pi'(u_j, u'_j) = \pi'(f(j)) = \theta'(j).
\end{align*}
Thus, $\theta$ and $\theta'$ are monotonic. This proves that $\theta$ and $\theta'$ are warping functions. 

It remains to show $\phi \circ \theta = \phi' \circ \theta'$. From Lemma \ref{lemma:psi=psi'} follows $\phi \circ \pi = \phi' \circ \pi'$. Then we have 
\[
\phi \circ \theta = (\phi \circ \pi) \circ f = (\phi' \circ \pi') \circ f = \phi' \circ \theta'.
\]
This completes the proof.
\end{proof}

\subsection{Semi-Metrification of DTW-Spaces}\label{subsec:semi-metrification}

In this section, we convert the dtw-distance to a semi-metric. For this, we regard time series as words over the alphabet $\S{A} = \R$. The set of finite time series is denoted by $\S{T} = \S{A}^*$. The next result shows that expansions of words over numbers can be expressed by matrix multiplication. 

\begin{lemma}\label{lemma:x>y=>x=PHIy}
Let $x, y \in \S{T}$ be two time series. Then the following statements are equivalent:
\begin{enumerate}
\item $x$ is an expansion of $y$.
\item There is a warping matrix $\Phi$ such that $x = \Phi y$.
\end{enumerate}
\end{lemma}

\begin{proof}
Suppose that $\abs{x} = n$ and $y = (y_1, \ldots, y_m)$. 

\medskip

\noindent
$\Rightarrow$: We assume that $x \succ y$. According to Lemma \ref{lemma:expansions} there are positive integers $\alpha_1, \ldots, \alpha_m \in \N$ such that $n = \alpha_1 + \cdots + \alpha_m$ and $x = y_1^{\alpha_1} \cdots y_m^{\alpha_m}$. Consider the function $\phi:[n] \rightarrow [m]$ defined by
\[
\phi(i) = \begin{cases}
1 & i \leq \alpha_1\\
2 & \alpha_1 < i \leq \alpha_1 + \alpha_2\\
\cdots & \cdots \\
m & \alpha_1 + \cdots \alpha_{m-1}< i 
\end{cases}
\]
for all $i \in [n]$. The function $\phi$ is surjective: Suppose that $j \in [m]$. We set $i = \alpha_1 + \cdots + \alpha_j$. Then $1 \leq i \leq n$ and $\phi(i) = j$ by definition of $\phi$. By construction, the function $\phi$ is also monotonically increasing. Hence, $\phi$ is a warping function. Then from Lemma \ref{lemma:expansion-by-phi} follows the second statement. 

\medskip

\noindent
$\Leftarrow$: Let $\Phi \in \B^{n \times m}$ be a warping matrix such that $x = \Phi y$. Then there is a warping function $\phi: [n] \rightarrow [m]$ associated with $\Phi$. The first statement follows by first applying Lemma \ref{lemma:expansion-by-phi} and then by Lemma \ref{lemma:expansions}.
\end{proof}

\medskip

\emph{Warping identification} is a relation on $\S{T}$ defined by $x \sim y \,\Leftrightarrow\, \dtw(x, y) = 0$ for all $x, y \in \S{T}$. We show that warping identification is an equivalence relation as claimed in Prop.~\ref{prop:warping-identification-class}.

\paragraph*{Section \ref{subsec:results:approach}, Prop.~\ref{prop:warping-identification-class}.}
The warping-identification $\sim$ is an equivalence relation on $\S{T}$. 

\begin{proof}
The relation $\sim$ is reflexive and symmetric by the properties of the dtw-distance. It remains to show that the warping-identification is transitive. Let $x, y, z \in \S{T}$ be time series with $x \sim y$ and $y \sim z$. Then $\dtw(x, y) = \dtw(y, z) = 0$. Hence, there are optimal warping paths $(\Phi, \Psi)$ and $(\Phi', \Psi')$ of length $\ell$ and $\ell'$, resp., such that $\norm{\Phi x - \Psi y} = \norm{\Phi'y - \Psi' z} = 0$. From Lemma \ref{lemma:pullback} follows that there are warping matrices $\Theta$ and $\Theta'$ of the same length $\ell$ such that $\Theta\Psi y = \Theta'\Psi'y$. Observe that
\begin{align*}
\norm{\Theta\Phi x - \Theta' \Psi' z} 
&= \norm{\Theta\Phi x - \Theta\Psi y + \Theta'\Phi' y - \Theta' \Psi' z}\\
&\leq \norm{\Theta\Phi x - \Theta\Psi y} + \norm{\Theta'\Phi' y - \Theta' \Psi' z}\\
&\leq \norm{\Theta}\norm{\Phi x - \Psi y} + \norm{\Theta'}\norm{\Phi'y - \Psi' z}\\
&= 0.
\end{align*}
Note that the second inequality in the third line follows from the fact that the Frobenius norm on matrices is compatible to the vector norm. From Lemma \ref{lemma:composition-01} follows that the products $\Theta\Phi$ and $\Theta' \Psi'$ are warping matrices. Since both products have the same length $\ell$, we find that the pair $(\Theta\Phi, \Theta' \Psi')$ is a warping walk. Then from Prop.~\ref{prop:dtw} follows that $\dtw(x, z) \leq \norm{\Theta\Phi x - \Theta' \Psi' z} = 0$. This proves that $\sim$ is transitive and completes the proof.
\end{proof}

For every $x \in \S{T}$ let $[x] = \cbrace{y \in \S{T} \,:\, x \sim y}$ denote the \emph{warping-identification class} of $x$. The next result presents an equivalent definition of the warping-identification class. 

\paragraph*{Section \ref{subsec:results:approach}, Prop.~\ref{prop:generator-of-[x]}}.
Let $x \in \S{T}$ be a time series with condensed form $x^*$. Then the warping-identification class of $x$ is of the form
\[
[x] = \cbrace{y \in \S{T} \,:\, y \succ x^*}.
\]

\begin{proof}
The warping-identification class is defined by 
\[
[x] = \cbrace{y \in \S{T} \,:\, x \sim y} = \cbrace{y \in \S{T} \,:\, \dtw(x,y) = 0}.
\] 
Let $\S{E}(x^*) = \cbrace{y \in \S{T} \,:\, y \succ x^*}$ denote the set of expansions of $x^*$. We show that $[x] = \S{E}(x^*)$.

\medskip

\noindent
$\subseteq$: Let $y \in [x]$. There is an optimal warping path $(\Phi, \Psi)$ such that $\dtw(x, y) = \norm{\Phi x - \Psi y} = 0$. From Lemma \ref{lemma:x>y=>x=PHIy} follows that there is a warping matrix $\Theta$ with $x = \Theta x^*$. Hence, $\norm{\Phi\Theta x^* - \Psi y} = 0$ and we obtain $\Phi\Theta x^* = \Psi y$. From Lemma \ref{lemma:composition-01} follows that the product $\Phi\Theta$ of warping matrices $\Phi$ and $\Theta$ is a warping matrix. We set $z = \Phi\Theta x^* = \Psi y$. Then $z \succ x^*$ and $z \succ y$ by Lemma \ref{lemma:x>y=>x=PHIy}. From Prop.~\ref{prop:compression-is-expansion-of-condensed-form} follows that $y \succ x^*$. This shows that $y \in \S{E}(x^*)$.

\medskip

\noindent
$\supseteq$: Let $y \in \S{E}(x^*)$. We assume that $\abs{x} = n$, $\abs{y}=m$, and $\abs{x^*} = k$. From Lemma \ref{lemma:x>y=>x=PHIy} follows that there are warping matrices $\Phi \in \B^{n \times k}$ and $\Psi \in \B^{m \times k}$ with $x = \Phi x^*$ and $y = \Psi x^*$, respectively. Then $(\Phi, I_n)$ and $(\Psi, I_{m})$ are warping walks of length $n$ and $m$ respectively. We have
\begin{align*}
0 \leq \dtw(x, x^*) &\leq \norm{I_n x - \Phi x^*} = 0 \\
0 \leq \dtw(y, x^*) &\leq \norm{I_{m} y - \Psi x^*} = 0
\end{align*}
giving $\dtw(x, x^*) = \dtw(y, x^*) = 0$. Hence, we have $x \sim x^*$ and $y \sim x^*$. Since $\sim$ is an equivalence relation, we have $x \sim y$ by Prop.~\ref{prop:warping-identification-class}. This proves $y \in [x]$. 
\end{proof}

\medskip

\noindent
We prove Prop.~\ref{prop:expansion-inequality} from Section \ref{subsec:results:approach}.

\paragraph*{Section \ref{subsec:results:approach}, Prop.~\ref{prop:expansion-inequality}.}
Let $x,y \in \S{T}$ be time series such that $x \succ y$. Then 
\begin{enumerate}
\itemsep0em
\item $\dtw(x, y) = 0$ 
\item $\dtw(x, z) \geq \dtw(y, z)$ for all $z \in \S{T}$. 
\end{enumerate}

\begin{proof}
We first show the second assertion. Let $z \in \S{T}$ be a time series and let $(\Phi, \Psi)$ be an optimal warping path between $x$ and $z$. Then we have $\delta(x, z) = \norm{\Phi x - \Psi z}$ by Prop.~\ref{prop:dtw}. From $x \succ y$ and Lemma \ref{lemma:x>y=>x=PHIy} follows that there is a warping matrix $\Theta$ such that $x = \Theta y$. We obtain
\begin{align*}
\delta(x, z) = \norm{\Phi x - \Psi z} = \norm{\Phi \Theta y - \Psi z} \geq \delta(y, z).
\end{align*}
From Lemma \ref{lemma:composition-02} follows that $(\Phi\Theta, \Psi)$ is a warping walk. The inequality holds, because $(\Phi\Theta, \Psi)$ is not necessarily an optimal warping walk between $y$ and $z$. 

The first assertion follows from the second one by setting $z = x$. We obtain
\[
0 = \dtw(x, x) \geq \dtw(y, x) \geq 0. 
\]
This implies $\dtw(y, x) = 0$ and completes the proof.
\end{proof}

The set $\S{T}^* = \cbrace{[x] \,:\, x \in \S{T}}$ is the quotient space of $\S{T}$ under warping identification $\sim$. We define the distance function 
\[
\delta^*: \S{T}^* \times \S{T}^* \rightarrow \R_{\geq 0}, \quad ([x], [y]) \mapsto \inf_{x' \in [x]}\; \inf_{y' \in [y]}\; \delta(x', y').
\]
We call $\delta^*$ the \emph{time-warp-invariant} (twi) distance induced by $\delta$. 

\paragraph*{Section \ref{subsec:results:approach}, Theorem \ref{theorem:semi-metric}.}
The twi-distance $\delta^*$ induced by the dtw-distance $\delta$ is a well-defined semi-metric satisfying $\delta^*([x], [y]) = \delta(x^*,y^*)$ for all $x, y \in \S{T}$. 

\begin{proof}
Let $x^*$ and $y^*$ be the condensed forms of $x$ and $y$, respectively. Repeatedly applying Prop.~\ref{prop:expansion-inequality} gives 
\begin{align*}
\dtw(x^*,y^*) \leq \dtw(x^*,y') \leq \dtw(x',y')
\end{align*}
for all $x' \in [x]$ and all $y' \in [y]$. Hence, the infimum exists and $\dtw^*([x], [y]) = \dtw(x^*, y^*)$ is well-defined. 

We show that $\delta^*$ is a semi-metric. Non-negativity and symmetry of $\dtw^*$ follow from non-negativity and symmetry of $\dtw$. To show the identity of indiscernibles, we assume that $\dtw^*([x], [y]) = 0$. From the identity $\delta^*([x], [y]) = \delta(x^*,y^*)$ follows $\dtw(x^*, y^*) = 0$. This implies that $x^*$ and $y^*$ are warping identical, that is $x^* \sim y^*$. By Prop.~\ref{prop:warping-identification-class} we have the following relations $[x^*] = [y^*]$, $[x] = [x^*]$, and $[y] = [y^*]$. Combining these relations gives $[x] = [y]$. This shows that $\dtw^*$ satisfies the identity of indiscernibles. Hence, $\delta^*$ is a semi-metric.
\end{proof}

\subsection{Warping-Invariance}\label{subsec:warping-invariance}

This section shows that the canonical extension of the twi-distance is warping-invariant. A distance function $d:\S{T} \times \S{T} \rightarrow \R_{\geq 0}$ is \emph{warping-invariant} if 
\[
d(x, y) = d(x', y')
\]
for all time series $x, y, x', y' \in \S{T}$ with $x \prec x'$ and $y \prec y'$. The twi-distance $\delta^*$ extends to a distance on $\S{T}$ by virtue of
\[
\delta^\sim: \S{T} \times \S{T} \rightarrow \R_{\geq 0}, \quad (x, y) \mapsto \delta^*([x], [y]).
\]
We call $\dtw^\sim$ the \emph{canonical extension} of $\dtw^*$.

\paragraph*{Section \ref{subsec:results:approach}, Theorem \ref{theorem:warping-invariance}.}
The canonical extension $\dtw^\sim$ of the twi-distance $\dtw^*$ is warping-invariant.

\begin{proof}
Let $x, x', y, y' \in \S{T}$ be time series such that there are common compressions $u \prec x, x'$ and $v \prec y, y'$. We show that $\delta^\sim(x, y) = \delta^\sim(x', y')$. Suppose that $x^*$ and $y^*$ are the condensed forms of $u$ and $v$, respectively. By repeatedly applying Prop.~\ref{prop:expansions-are-transitive} we obtain $x, x' \succ x^*$ and $y, y' \succ y^*$. From Prop.~\ref{prop:generator-of-[x]} follows that $[x] = [x']$ and $[y] = [y']$. This gives $\delta^\sim(x, y) = \delta^*([x], [y]) = \delta^*([x'], [y']) = \delta^\sim(x', y')$. The proof is complete. 
\end{proof}

\commentout{
\begin{proposition}\label{prop:equivalent-formulation-warping-invariance}
Let $d:\S{T} \times \S{T} \rightarrow \R_{\geq 0}$ be a distance function. Then the following statements are equivalent:
\begin{enumerate}
\item $d$ is warping-invariant.
\item $d(x, y) = d(x', y')$ for all $x, y, x', y' \in \S{T}$ with $x \sim x'$ and $y \sim y'$. 
\end{enumerate}
\end{proposition}
\begin{proof}
We first assume that $d$ is warping-invariant. Let $x, y, x', y' \in \S{T}$ be time series such that $x \sim x'$ and $y \sim y'$. Proposition~\ref{prop:warping-identification-class} gives $[x] = [x']$ and $[y] = [y']$. From Prop.~\ref{prop:generator-of-[x]} follows that there are condensed forms $x^*$ and $y^*$ that generate the equivalence classes $[x] = [x']$ and $[y] = [y']$, respectively. Thus we have $x^* \prec x, x'$ and $y^* \prec y, y'$. Since $d$ is warping-invariant, we have 
\[
d(x^*, y^*) = d(x, y) \quad \text{ and } \quad d(x^*, y^*) = d(x', y').
\]
Combining both equations yields $d(x, y) = d(x', y')$. This shows that the first statement implies the second.

\medskip

We assume that the second statement holds. Let $x, y, x', y' \in \S{T}$ be time series such that $x \prec x'$ and $y \prec y'$. From Prop.~\ref{prop:expansion-inequality} follows that $\dtw(x, x') = 0$ and $\dtw(y, y') = 0$. Hence, we find that $x \sim x'$ and $y \sim y'$. This implies $d(x, y) = d(x', y')$ by assumption. Hence, the first statement holds. 
\end{proof}
}

\end{document}